\documentclass{article}

\usepackage{microtype}
\usepackage{graphicx}
\usepackage{subcaption}
\usepackage{booktabs} 

\usepackage{hyperref}

\usepackage[preprint]{icml2026}

\usepackage{amsmath}
\usepackage{amssymb}
\usepackage{mathtools}
\usepackage{amsthm}
\usepackage{algorithm}
\usepackage{algorithmic}
\usepackage{abbe}
\usepackage{subcaption}
\usepackage{thmtools}

\usepackage[capitalize,noabbrev]{cleveref}

\theoremstyle{plain}
\newtheorem{theorem}{Theorem}

\newtheorem{lemma}[theorem]{Lemma}

\theoremstyle{definition}
\newtheorem{definition}[theorem]{Definition}

\theoremstyle{remark}

\usepackage[textsize=tiny]{todonotes}

\icmltitlerunning{Dominant Arm Identification with Mixing and Recycling Observed Samples}

\begin{document}

\twocolumn[
  \icmltitle{Dominant Arm Identification with Mixing and Recycling Observed Samples}



\icmlsetsymbol{equal}{*}
\begin{icmlauthorlist}
    \icmlauthor{Jonghyun Sim}{cau}
    \icmlauthor{Wonyoung Kim}{cau}
\end{icmlauthorlist}

\icmlaffiliation{cau}{Chung-Ang University, Seoul, Republic of Korea}

\icmlcorrespondingauthor{Wonyoung Kim}{wyk7@cau.ac.kr}


  \vskip 0.3in
]



\printAffiliationsAndNotice{}  

\begin{abstract}
Classical best-arm identification ranks arms by their mean rewards. Under this criterion, arms with the same mean are equally desirable, even if some yield consistent rewards while others attain that mean
through rare, large rewards.
This motivates comparing arms based on where their realized rewards fall within the reward distributions of all arms.
We therefore introduce a novel dominant arm criterion
and an efficient estimator with theoretical guarantees. Our approach relies on two key technical innovations: (i) a dominance score criterion that an arm beats the locally dominant arm over the partitioned reward space and (ii) a joint mixing and recycling mechanism coupled with a doubly robust estimator that guarantees simultaneous convergence of the empirical distribution functions for all arms. These key innovations pave the way to efficient computation of global arm dominance. Our proposed elimination algorithm identifies the dominant arm with a near-optimal rate of sample complexity. Numerical experiments demonstrate that our algorithm consistently achieves exact recovery of the true dominant arm, outperforming existing baselines.
\end{abstract}

\section{Introduction}
Identifying the best option under uncertainty is a fundamental problem in statistical decision-making.
It has broad applications in online experimentation
\citep{scott2015multiarmed}, personalized recommendation
\citep{li2010contextual}, adaptive treatment allocation
\citep{tomkins2021intelligent}, and portfolio selection
\citep{huo2017risk}.
In these settings, a decision-maker observes noisy outcomes from multiple arms and aims to identify the most preferable arm according to a suitable comparison criterion.
However, the notion of the ``best'' arm is not always obvious, especially when reward distributions differ not only in their means but also in their tail behavior.

Existing comparison criteria include mean rewards, risk-sensitive objectives,
quantile-based criteria, and pairwise winning probabilities.
An arm with the highest mean reward can still be outperformed on many
individual draws. The mean alone does not indicate which arm is most
likely to exceed a given reward threshold.
Pairwise comparisons, in turn, may not yield a globally coherent ordering
among multiple arms. These limitations become particularly important when
reward distributions have crossing cumulative distribution functions (CDFs)
or exhibit heterogeneous tail behavior.

To compare $K$ arms with rewards $R_1,\ldots,R_K$ and unknown reward
distributions, we may consider the probability that an arm's realized
reward exceeds those of all other arms in the same round.
If rewards are independent across arms and have continuous CDFs
$F_1,\ldots,F_K$, conditioning on the reward from arm $a$ gives
$W_a
= \PP\bigl(R_a > \max_{k\neq a} R_k\bigr)
= \int_{\RR}\prod_{k\neq a}F_k(r)\,dP_a(r),$
where $P_a$ is the reward distribution of arm $a$.
This joint winning probability is useful in a single round and has many applications including real-time bidding \citep{ghosh2020scalable}, random utility models \citep{cascetta2009dominance}, and solver benchmarking \citep{dolan2002benchmarking}.

However, the joint winning probability has two limitations.
(i) Dependence on $K$ is unavoidable. 
Replacing each $F_k$ with its estimate $\hat F_k$ in the
integrand $\prod_{k\neq a}F_k(x)$ of the winning probability
$W_a$ yields the following error bound by the mean value theorem:
\[
\sup_x\left|\prod_{k\neq a}\hat F_k(x) - \prod_{k\neq a}F_k(x)\right| \;\le\; (K-1)\max_k\|\hat F_k-F_k\|_\infty.
\]
The error bound thus grows linearly with $K$ for a fixed per-arm CDF accuracy. 
(ii) The winning probability records \emph{how often} an arm wins,
but does not capture \emph{how large} its rewards are when it loses.
When the winning probability is close to $1/K$, the best winning probability arm has the extremely large value with a small probability which may not be reliable.
In many practical applications, an arm that consistently yields
relatively high rewards may be preferred over such an arm.
Examples include cancer treatment selection, where consistently
favorable survival outcomes are desirable \citep{hong2025value},
semiconductor manufacturing, where stable and high production yields
are important \citep{xu2024fast}, and large-scale computing
systems, where fast and predictable response times are valued
\citep{dean2013tail}.
Consider the following independent distributions:
\[
\begin{split}
A&:\ \{10,\,11\},\quad \mathbb{P}(A=11)=\tfrac12,\\
B&:\ \{1,\,20\},\quad \mathbb{P}(B=20)=\tfrac13,\\
\qquad
C&:\ \{2,\,15\},\quad \mathbb{P}(C=15)=\tfrac12.
\end{split}
\]
The winning probabilities are all equal $W_A = W_B = W_C = 1/3$, but $B$ wins only via an extreme, all-dominating realization ($B=20$, which occurs with probability $\tfrac13$ and beats every possible outcome outright), while in its remaining $\tfrac23$ of outcomes it takes the single worst value across all arms ($B=1$). Arms $A$ and $C$, by contrast, reach their own best value with higher probability ($\tfrac12$) even though that best value does not guarantee an outright win --- they achieve the same win rate through more consistent high realizations rather than one rare, maximal one. The winning probability cannot see this difference because it only records the binary event of being best.
The winning probability overlooks the local structure when the arm is not the best and some extreme values lead to the overestimation of the winning probability.
To see what the winning probability leaves out, we consider where a
non-winning reward falls within each arm's reward distribution.
Even when arm $A$ loses with a reward of $10$, that reward is at least
as high as half of the rewards from each arm. By comparison, a reward
of $1$ from arm $B$ falls below every possible reward from either $A$
or $C$.
More generally, we can assess the relative standing of any realized
reward $r$ by comparing it with the distribution most likely to exceed
$r$. When the reward CDFs cross, this comparison distribution can
change with $r$.
Our dominance score $D_a$ averages these comparisons over the reward
distribution of arm $a$.

Our main contributions are summarized as follows:
\begin{itemize}
    \item We propose a novel dominant arm criterion based on local distributional dominance (Definition~\ref{def:dominance}). 
    The proposed dominance score generalizes the intuition of stochastic dominance by partitioning the reward space and comparing each arm with the locally dominant arm.
    We provide a lower bound for the sample complexity to identify the arm with the highest dominance score. (Theorem~\ref{thm:sample_complexity_lower_bound}).
    \item We develop a novel estimator for the distribution function that converges simultaneously over all candidate arms for each round (Corollary~\ref{cor:uniform_estimator_rate}).
    Our estimator gains information about all arms by mixing the recycled samples and overcomes the dependency among samples by constructing the doubly robust estimator (Theorem~\ref{thm:dr_estimator_error_bound}).
    \item We propose a doubly robust dominance-score elimination (\texttt{DRDSE}) algorithm (Algorithm~\ref{alg:DRDSE}) that correctly identifies the dominant arm with a guaranteed near-optimal sample complexity bound (Theorem~\ref{thm:sample_complexity}).
    \item We provide experimental results demonstrating that the proposed algorithm correctly identifies the dominant arm using fewer samples.
\end{itemize}

\section{Related Work}
Classical best-arm identification in multi-armed bandits typically defines the best arm as the arm with the largest mean reward. Early pure-exploration methods study how to identify an $\epsilon$-optimal arm with high probability using confidence-based elimination procedures \citep{evendar2006action}, and subsequent work develops finite-sample guarantees for best-arm identification under fixed-budget or fixed-confidence settings \citep{audibert2010best,kaufmann2016complexity}. 
While mean-based comparison is simple and statistically tractable, it may fail to capture important distributional features when the rewards have crossing CDFs.

To address these limitations, several alternative comparison criteria have been studied. Mean-variance and risk-averse bandit formulations compare arms by incorporating dispersion or risk-return tradeoffs rather than expected reward alone \citep{sani2012risk,vakili2016risk,zhu2020thompson}. 
Quantile-based approaches compare arms through distributional functionals such as upper or lower quantiles, allowing the learner to focus on tail behavior or qualitative reward levels \citep{szorenyi2015qualitative,zhang2021quantile,lau2025quantile}. Pairwise winning criteria arise in dueling bandits and preference-based learning, where the learner observes noisy comparisons between pairs of arms instead of absolute rewards \citep{yue2012k,bengs2021preference}. 
These criteria provide richer notions of comparison than the mean, but each focuses on a specific quantile or pairwise preference structure. 
In contrast, our proposed dominance score compares arms by comparing with the locally dominant arm over the partitioned region.

Our estimation procedure is related to the literature on semiparametric efficiency, where auxiliary models are introduced to achieve greater efficiency than nonparametric estimation~\citep{bickel1993efficient,vaart1998asymptotic,tsiatis2006semiparametric}. 
The widely used auxiliary models are probability selection model and the imputation model, which are related to inverse-propensity weighting and doubly robust estimation. 
Doubly robust estimators combine weighting and outcome regression, and remain consistent when either the propensity model or the outcome model is correctly specified \citep{bang2005doubly}. 
Doubly robust estimation has also been applied in sequential decision-making and bandit problems under adaptive data collection \citep{kim2021doubly}. 
Our work builds on this line of research by developing a doubly robust estimator for the proposed distributional dominance score and establishing its finite-sample concentration guarantees.

\section{Problem Formulation}

We consider a stochastic multi-armed bandit problem with $K$ arms. 
At each round $t$, a decision-maker selects an arm $a_t \in [K]:=\{1, \ldots, K\}$ and observes a reward $R_{a_t, t}$, drawn independently from an unknown 
distribution $\mathbb{P}_a$.
The goal is to identify the best arm according to a suitable optimality criterion.

A natural criterion for the best arm is the highest mean reward which leads to the classical best-arm identification problem. 
However, mean-based criteria may fail to capture important distributional features and the highest realized reward.
Alternative criteria include quantile-based objectives, risk-sensitive measures, and pairwise winning probabilities $\mathbb{P}(R_a > R_{a'})$ for $a \neq a'$.
While pairwise winning probabilities provide a richer notion of comparison 
than the mean, estimating $\mathbb{P}(R_a > R_{a'})$ for all pairs 
requires sample complexity that scales proportionally to $K^2$, as each of 
the $\binom{K}{2}$ pairs must be estimated separately. 
Furthermore, pairwise comparisons may not yield a globally coherent ordering among arms. 
This motivates a criterion that aggregates distributional information across 
all arms simultaneously, without requiring pairwise estimation.

We propose the following dominance score for the dominant arm identification.
\begin{definition}
\label{def:dominance}
We define the dominance score of arm $a$ as:
\[
D_{a} := \EE[\min_{k\in[K]}F_k(R_a)]
\]
\end{definition}
Compared to the winning probability $W_a$ that involves the multiplication of $K-1$ CDFs, the error bound for $D_a$
involves the maximum of $K$ CDF estimation errors without an additional dependence on $K$.
We show that the definition measures how arm $a$ performs relative to the locally dominant arms across the reward space.
For each $k\in[K]$, let $\mathcal{G}_{k} = \{r \in \mathbb{R} : \mathbb{P}(R_{k} > r) \ge \max_{k^{\prime} \neq k} \mathbb{P}(R_{k^{\prime}} > r) \}$ denote the region where the distribution of action $k$ dominates all other arms. Ties are broken arbitrarily such that $\mathcal{G}_{k} \cap \mathcal{G}_{k^{\prime}} = \emptyset$ for $k \neq k^{\prime}$, ensuring that $\{ \mathcal{G}_{k} \}_{k=1}^K$ forms a partition of $\mathbb{R}$.

\begin{lemma}[Local dominance formula]
\label{lem:dominance_to_local}
The dominance score of arm $a$ is defined as:
\begin{equation*}
\begin{aligned}
    \EE[\min_{k}F_k(R_a)] &= 
    \sum_{k=1}^{K} \mathbb{P}(R_{k} \le R_{a}, R_{a} \in \mathcal{G}_{k}) \\
    &= \sum_{k=1}^{K} \int_{\mathcal{G}_{k}} \mathbb{P}(R_{k} \le r) \, d\mathbb{P}_{a}(r).
    \label{eq:dominance_score}
\end{aligned}
\end{equation*}
\end{lemma}

The dominance score represents the performance of an action that transcends simple mean-based comparisons by evaluating the dominance based on the distributional properties across the entire reward space. 
Our approach begins by partitioning the real line into disjoint regions where each specific arm is stochastically dominant, meaning that in a given region, a particular arm has the highest probability of exceeding a threshold compared to its peers. 
By defining the dominance score as the sum of the probabilities that arm $a$ outperforms the locally dominant arm within each partition, the metric captures nuanced behavior in the tails and segments of the reward distribution that traditional expected values might overlook. 
This formulation is particularly advantageous in settings with crossing CDFs, as it relies on the relative ranks of rewards rather than their absolute magnitudes, providing a scale-invariant objective.
Ultimately, the dominance score provides a comprehensive framework for multi-armed bandit problems, where the goal is to identify arms that are not just superior on average but are consistently competitive across the diverse realized values of the reward space.

Next, we provide a problem-dependent lower bound for the identification of the arm that maximizes the dominance score.
\begin{restatable}[Sample Complexity Lower Bound]{theorem}{thmSampleComplexityLowerBound}
\label{thm:sample_complexity_lower_bound}

Let $\alpha\in(0,\tfrac12)$ and suppose $\mathcal{M}\supseteq\mathcal{F}$. Let
$\nu=(\mathbb{P}^{(p_1)},\dots,\mathbb{P}^{(p_K)})$ with
$p_a\in[\alpha,1-\alpha]$ for all $a\in[K]$ and $p_1>\max_{a\neq1}p_a$.
Then for every $\delta\in(0,\tfrac12)$, every $\delta$-correct algorithm on
$\mathcal{M}$ satisfies
\[
\mathbb{E}_\nu[N_a(\tau)]\;\ge\;\frac{\alpha(1-\alpha)}{4\,\Delta_a^2}\,
\log\frac{1}{2.4\delta}
\qquad\text{for all } a\neq1.
\]
\end{restatable}
The lower bound of the sample complexity has the same rate as in the best-arm identification problem in terms of the sub-optimality gap $\Delta_a$.
In Theorem~\ref{thm:sample_complexity}, we show that our proposed algorithm matches the upper sample complexity bound up to logarithmic factors.

\section{Proposed Estimation Procedure}
\label{sec:semi_est}
In this section, we provide an estimation procedure for the proposed dominance score. 
Estimating $D_a$ requires a bound holding uniformly over $k\in[K]$ because both the partition $\{\Gcal_k\}_{k=1}^{K}$ and the value of each dominance score depend on the entire collection of CDFs $\{F_k\}_{k=1}^{K}$. 
To achieve the simultaneous convergence, we first construct an exploration-mixed estimator by combining the reward observed in the current exploitation round with reward samples recycled from the exploration phase. This mixing and recycling procedure allows each exploitation round to inform every coordinate of $(F_1(x),\ldots,F_K(x))^\top$, although the repeated use of the exploration samples introduces dependence among the mixed observations. 
Next, we construct a doubly robust estimator that is robust
to the error in the exploration-mixed estimator incurred
by the dependency among recycled samples.
Finally, based on the doubly robust empirical distribution functions, we construct the empirical partition and the corresponding dominance score estimator. 

\subsection{Distribution Estimation with Mixing and Recycling}
\label{subsec:mixing}
To estimate the cumulative distribution functions $F_k(x):=\PP(R_k<x)$ for $k\in[K]$, the empirical cumulative distribution is known as the standard choice.
Given a sequence of selected actions $a_1,\ldots,a_t$ and rewards $R_{a_1},\ldots,R_{a_t}$, the empirical cumulative distribution function $\hat{F}_{k,\mathrm{emp}}(x)=N_{k,t}^{-1}\sum_{s=1}^{t}\II(a_s=k)\II(R_{a_s}\le x)$, where $N_k(t) = \sum_{s=1}^t \mathbb{I}(a_s = k)$.
The convergence rate is $\Vert\hat{F}_{k,\mathrm{emp}} - F_k\Vert_{\infty} = \tilde{O}(1/\sqrt{N_{k,t}})$.
In the worst case, this leads to a total sample complexity that scales linearly with $K$ to ensure all arms are accurately estimated.

When we view this problem as estimating a $K$-dimensional vector $(F_1(x),\ldots,F_{K}(x))^\top$ with $K>3$, \citet{james1961estimation} supports that there exists a superefficient estimator that has better convergence rate than the average estimator.
We show in Corollary~\ref{cor:uniform_estimator_rate} that our proposed estimator $\widehat{F}_k$ achieves a \textit{simultaneous uniform rate} of:
\begin{equation*}
    \max_{k \in [K]} \Vert \widehat{F}_k - F_k\Vert_{\infty} = \tilde{O}\left(\frac{1}{\sqrt{t}}\right),
\label{eq:simulataneous_rate}
\end{equation*}
with $t$ samples.
Instead of relying only on the reward samples, we induce auxiliary models to improve efficiency of the estimator.
Our estimator incorporates two auxiliary models --- a reward model and a propensity model---to estimate the $K-1$ unselected (and thus missing) outcomes in each round. 
By exploiting the auxiliary models, the data-sharing mechanism in the DR estimator ensures that the effective sample size for every arm scales with the total number of rounds $t$, rather than being limited by its individual selection frequency.
Combining these two auxiliary models we construct a doubly robust estimator that achieves a better rate when estimating the $K$ distribution functions simultaneously.

We introduce the mixing and recycling scheme that enables the estimator to gain information over all $K$ arms for each round.
Given a failure probability $\delta\in(0,1)$, let $\Ecal$ denote the set of exploration rounds that satisfies $|\Ecal\cap[t]|=\lceil 8K\log\frac{Kt^2}{\delta}\rceil$
for $t\ge T_{e}$ and $[T_{e}]\subset\Ecal$, where $T_{e}=\inf\{t\ge1:t\ge 8K\log \frac{Kt^2}{\delta}\}=O(K\log K)$.
For the exploration round $t\in\Ecal$, we sample all $K$ arms for
the round $t,t+1,\ldots,t+K-1$. For the exploitation round $t\in\Ecal^{c}$,
we sample an undetermined arm $a_{t}$. 
Conventional estimators do not gain any information on the rewards of arms in $[K]\setminus\{a_{t}\}$.
In contrast, our estimator gains information by recycling the data in the exploration phase $\Ecal\cap[t]$ for the imputation of the unselected rewards. 
Define $\tau_{t}$ as the minimally re-used round for the exploitation phase $t\in\Ecal^{c}$, i.e., $\tau_{t}=\min_{\tau}\{\argmin{\tau\in\Ecal\cap[t]}\sum_{s\in[t-1]\setminus\Ecal}\II(\tau=\tau_{s})\}$.
By the minimization and the construction of $\Ecal$, we have $a_{\tau_{t}+k-1}=k$ for $k\in[K]$. 
Then for mean zero and unit variance IID $m_{1,t},\ldots,m_{K,t}\sim U[-\sqrt{3},\sqrt{3}]$ mixing variables, we define the mixed contexts and rewards,
\begin{equation*}
\begin{split}
\tilde{X}_{a_t,t}
&:=
\frac{1}{\sqrt{K}}
(m_{1,t},\ldots,m_{K,t})^{\top}
+\eb_{a_t},\\
\tilde{R}_{a_t,t}(x)
&:=
\frac{1}{\sqrt{K}}
\sum_{k=1}^{K}
m_{k,t}
\II\!\left(
R_{a_{\tau_t+k-1},\,\tau_t+k-1}\le x
\right)\\
&\quad+
\II(R_{a_t,t}\le x).
\end{split}
\end{equation*}

At round $t$, we sample the pseudo-actions $\tilde{a}_t(1),\ldots,\tilde{a}_t(\rho_t)$ from the distribution $\PP(\tilde{a}_{t}=a_{t})=1/2$ and $\PP(\tilde{a}_{t}=a)=1/(2K-2)$ for all $a\in[K]\setminus\{a_{t}\}$.
The probability $\PP(\tilde{a}_{t}=a_{t})$ can be set in $(0,1)$ other than $1/2$ to balance between the burn-in period and the error bound for the estimator. 
However, for simplicity, we set it as 1/2.
Let $\rho_t:=\min\{n\in\NN: \tilde{a}_t(n)=a_t\}$ denote the first trial that the pseudo-action matches the candidate action $a_t$.
$\Psi:=\{t\in\Ecal^c: \rho_t \le \lceil \frac{\log(t+1)}{\log 2} \rceil \}$
denotes the rounds that the matching happens within $\lceil \frac{\log(t+1)}{\log 2} \rceil$ trials, where the number is set to satisfy $\PP(t\in\Psi) \ge 1/(t+1)$.

Let $w_{s}=\PP(\tilde{a}_t=a_t)^{-1}\II(s\in\Psi)+\II(s\in\Psi^{c})$ denote the inverse probability weight for the matching rounds $\Psi$.
We define the exploration-mixed Gram matrix, $\tilde{V}_{t}:=\sum_{s\in[t]\cap\Ecal}w_{s}\eb_{a_{s}}\eb_{a_{s}}^{\top}+\sum_{s\in[t]\cap\Ecal^{c}}w_{s}\tilde{X}_{a_{s},s}\tilde{X}_{a_{s},s}^{\top}$.

\begin{restatable}[Spectrum Bounds for the Empirical Gram Matrix]{lemma}{lemMixedGramMatrix}
\label{lem:mixed_gram_matrix_spectrum}
Given failure probability $\delta$, for any large $t$ that satisfies $t\ge8K\log\frac{Kt^2}{\delta}$, with probability at least $1-10\delta/t^2$ the empirical Gram matrix $\tilde{V}_t \in \mathbb{R}^{K \times K}$ satisfies the positive semi-definite ordering:
\[
\frac{t}{6} \Ib_{K} \preceq \tilde{V}_{t} \preceq \frac{13t}{6} \Ib_{K}.
\]
\end{restatable}
Lemma~\ref{lem:mixed_gram_matrix_spectrum} shows the growth
rate of the exploration-mixed Gram matrix is $t$.
This growth rate is directly related to the convergence rate of the estimator.

We define the exploration-mixed CDF estimator:
\begin{equation}
\label{eq:exploration_mixed_estimator}
\begin{split}
\tilde{\Fb}_{t}(x):=\tilde{V}_{t}^{-1}\Big(
&\sum_{s\in[t]\cap\Ecal}w_{s}\II(R_{a_{s},s}\le x)\eb_{a_{s}} \\
&+\sum_{s\in[t]\cap\Ecal^{c}}w_{s}\tilde{R}_{a_{s},s}(x)\tilde{X}_{a_{s},s}\Big)
\end{split}
\end{equation}
The exploration-mixed estimator gains information about all arms at each round while estimating $K$ values of the distribution separately does not share the data across arms.
The mixed estimator has the following self-normalized bound:


\begin{restatable}[Error Bound for the Mixed Estimator]{theorem}{lemMixedEstimator}
\label{thm:mixed_estimator_error_bound}
For large $t$ such that $t\ge K\log\frac{Kt^2}{\delta}$, an arm $k$, and failure probability parameter $\delta \in (0, 1)$, the mixed estimator $\tilde{\Fb}_{t}:=(\tilde{F}_{1,t},\ldots,\tilde{F}_{K,t})$ satisfies:
\[
\|\tilde{F}_{k,t} - F_{k}\|_{\infty} \le 102 \sqrt{\frac{3}{t} \log\frac{Kt^{2}}{\delta}} + 294 \sqrt{\frac{\log\frac{2Kt^{2}}{\delta}}{|[t] \cap \mathcal{E}|}},
\]
for all $k\in[K]$, with probability at least $1 - \frac{14\delta}{t^2}$.
\end{restatable}

The bound has $\sqrt{1/|\Ecal\cap[t]|}$ convergence rate.
This loose bound is due to the dependency of the explored sample and mixed sample because of the recycling procedure, i.e., the duplicated usage of the explored sample.
Thus, we introduce the doubly robust estimation to obtain an estimator that gains the simultaneous data across all arms while being robust to the error caused by the dependency among samples.

\subsection{Doubly Robust Estimation}
\label{subsec:dr_est}

For $t\in\Psi$ and an imputation CDF $\tilde{\Fb}_t(x)$, we use a matching sample $\tilde{a}_t := \tilde{a}_t(\rho_t)$ and define
\[
\check{F}_{k,t}(x)=(1-\frac{\II(\tilde{a}_t=k)}{\PP(\tilde{a}_t=k)})\tilde{F}_{k,t}(x)+\frac{\II(\tilde{a}_t=k)}{\PP(\tilde{a}_t=k)}\II(R_{\tilde{a}_t,t} \le x).
\]
The imputed estimator $\tilde{F}_{k,t}$ is multiplied by the mean-zero random variable $1-\frac{\II(\tilde{a}_t=k)}{\PP(\tilde{a}_t=k)}$ to offset the error incurred by the imputation estimator.
This pseudo distribution is computable for $t\in\Psi$ where the pseudo-action $\tilde{a}_t$ and the chosen action $a_t$ match, which yields,
\[
\check{F}_{k,t}(x)= \tilde{F}_{k,t}(x)
+2\mathbb{I}(a_t=k)
\left[
\mathbb{I}(R_{a_t,t}\le x)-\tilde{F}_{k,t}(x)
\right].
\]
The DR estimator is defined as,
\[
\widehat{F}_{k,t}(x)
:=\frac{\underset{s\in[t]\cap\Psi}{\sum}\check{F}_{k,s}(x)
+\underset{s\in[t]\setminus\Psi}{\sum} \II(a_s=k,R_{a_s,s}\le x)}{|[t]\cap\Psi|+\underset{s\in[t]\setminus\Psi}{\sum}\II(a_s=k)}.
\]

We provide the error bound of the doubly robust estimator in terms of the error of the imputed estimator.
\begin{restatable}[Error Bound for the Doubly Robust Estimator]{theorem}{thmDREstimatorErrorBound}
\label{thm:dr_estimator_error_bound}
For large $t$ such that $t\ge K\log\frac{Kt^2}{\delta}$ and an arm $k \in [K]$, let $\widehat{F}_{k,t}$ denote the doubly robust (DR) estimator constructed using $\tilde{F}_{k,t}$ as the outcome imputation model. Then, for any failure probability parameter $\delta \in (0, 1)$, the uniform error satisfies:
\begin{equation*}
\begin{split}
\left\|\widehat{F}_{k,t}-F_k\right\|_{\infty}
&\le
9\sqrt{
\frac{1}{t}
\log\left(\frac{Kt^3}{\delta}\right)
}\\
&\quad+
12\sqrt{
\frac{2K}{t}
\log\left(\frac{Kt^2}{\delta}\right)
}
\left\|\tilde{F}_{k,t}-F_k\right\|_{\infty},
\end{split}
\end{equation*}
for all $k\in[K]$ with probability at least $1 - 7\delta/t^{2}$.
\end{restatable}

Theorem~\ref{thm:dr_estimator_error_bound} gives the error bound of the DR estimator in terms of the error of the imputation estimator normalized by the exploration-mixed Gram matrix $\tilde{V}_t$.
The term `doubly robust' refers to the robustness to the misspecification error of the imputation estimator $\tilde{F}_t(x)$ and the probability selection model $\PP(\tilde{a}_t=a)$.
The robustness takes effect when either of the two auxiliary models is correctly specified.
In our case, the probability model $\PP(\tilde{a}_t=a)$ is known and
the correctness of the probability selection model improves the efficiency of the estimator in that the error of the imputation estimator $\Vert\tilde{F}_{k,t}-F_k\Vert_{\infty}$ is multiplied by a vanishing $\sqrt{K/t}$ term.

\begin{restatable}[Uniform Rate of the Proposed Estimator]{corollary}{corUniformEstimatorRate}
\label{cor:uniform_estimator_rate}
For any round $t \ge T_e$ with sufficient exploration and an arm $k \in [K]$, the proposed estimator $\widehat{F}_{k,t}$ satisfies the uniform concentration bound:
\[
\left\|\widehat{F}_{k,t} - F_{k}\right\|_{\infty} \le \gamma_{t}:=1926 \sqrt{\frac{1}{t} \log\left(\frac{K t^{3}}{\delta}\right)},
\]
for all $k\in[K]$, with probability at least $1 - 21\delta/t^{2}$.
\end{restatable}

The proof is deferred to Appendix~\ref{app:F_conv_proof}.
The conventional statistical lower bound (e.g., \citet{scarlett2021fano}) implying $\Omega(\sqrt{K\log t/t})$ applies only to estimators restricted to the already-sampled observations; it does not rule out faster rates for estimators that exploit additional structure beyond the observed samples. 
Our proposed estimator exploits the two auxiliary models that affect the underlying mechanisms of observing the samples: (i) the imputation models $\check F_{k,t}$, each estimated from at least $O(\log t)$ samples per distribution, and (ii) the sampling probability $\PP(\tilde{a}_t = k)$, which governs whether arm $k$ is observed at round $t$ and is known to the algorithm by construction.

\citet{robins2008higher} showed that this knowledge of the observation probability enables estimators to surpass the rate achievable by any estimator that disregards it: when the propensity of observation is known, a suitably constructed estimator can achieve a strictly faster convergence rate than the conventional lower bound would suggest, because the known propensity supplies semiparametric information absent from the sampling mechanism alone. 
Our setting exactly satisfies this condition, and our estimator attains the corresponding improved rate.

\begin{figure*}[t]
    \centering
    \begin{subfigure}{0.49\textwidth}
        \includegraphics[width=\textwidth]{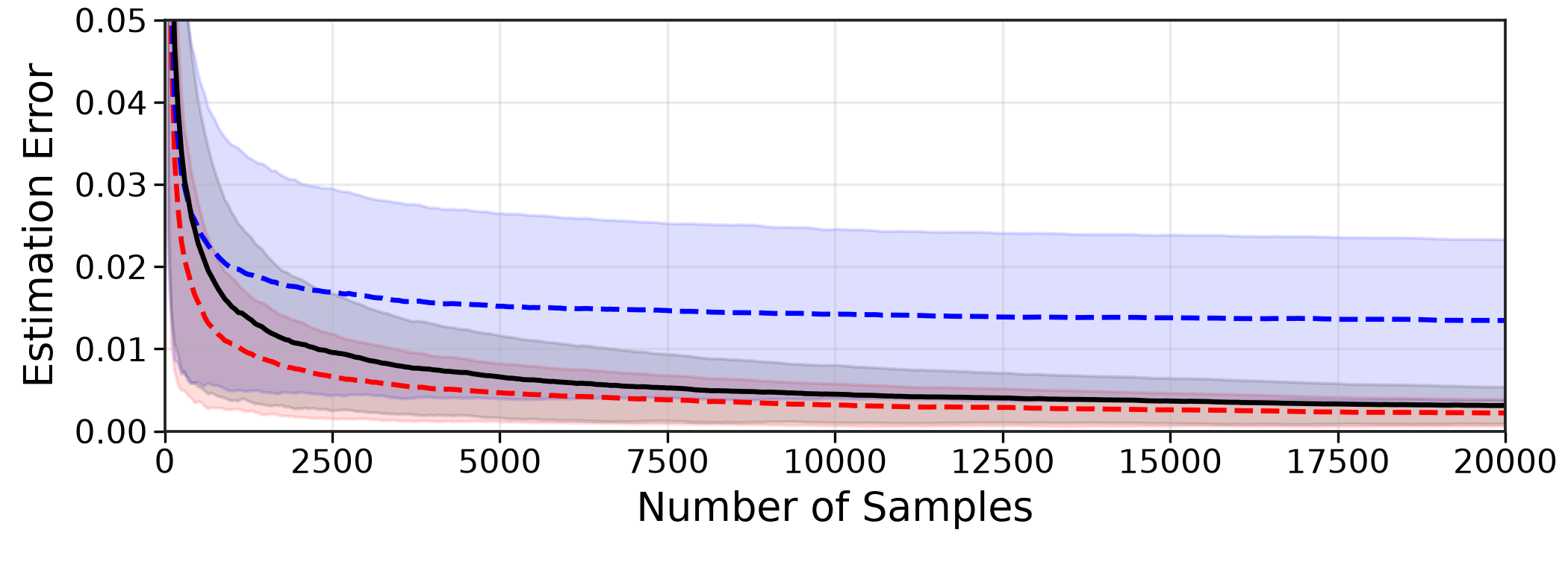}
        \caption{Error for the exploited arm (arm 1)}
    \end{subfigure}
    \hfill
    \begin{subfigure}{0.49\textwidth}
        \includegraphics[width=\textwidth]{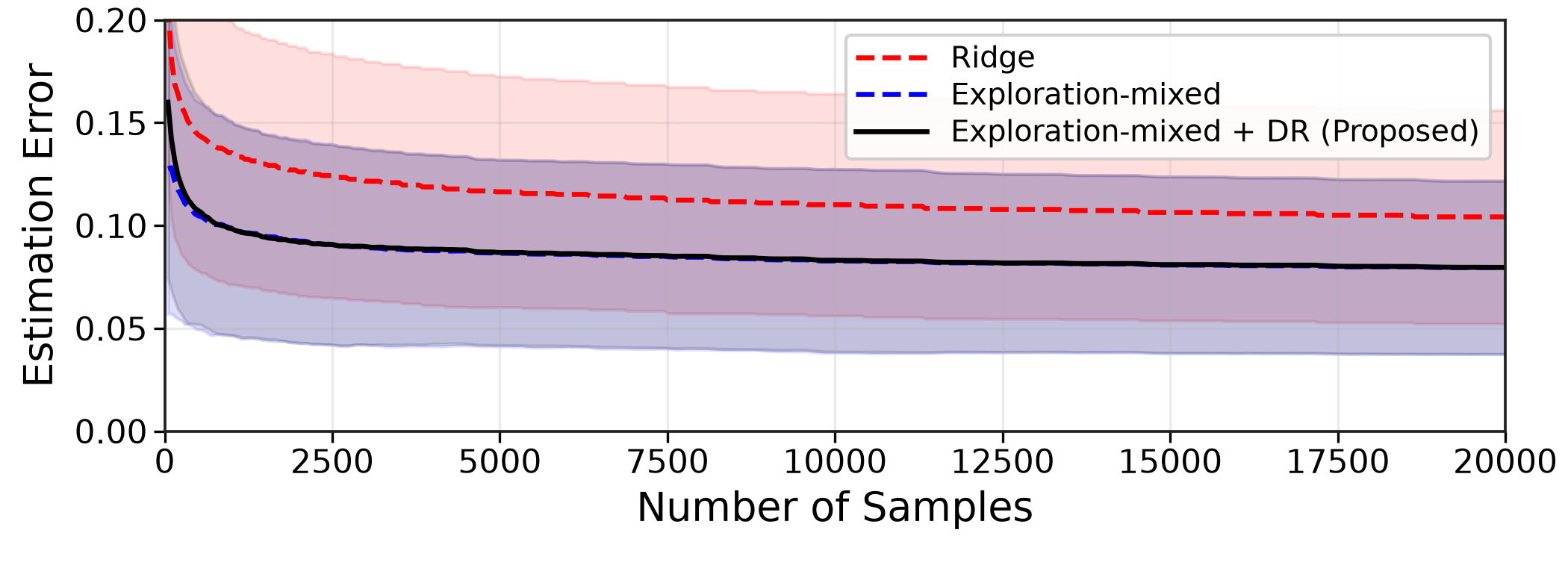}
        \caption{Error for the unselected arms (arms 2 and 3)}
    \end{subfigure}
    \caption{A comparison of CDF estimation errors of the conventional ridge,
    exploration-mixed, and proposed DR estimators for a 3-armed bandit.
    The lines and shaded regions represent the average and standard
    deviation, respectively, over \(1{,}000\) independent experiments.}
    \label{fig:cdf_error}
\end{figure*}
Figure~\ref{fig:cdf_error} shows the error of the ridge estimator, exploration-mixed estimator and the DR estimator.
We consider \(K=3\) heterogeneous arms with Bernoulli, Gaussian, and Pareto rewards. 
The sampling rule selects arm \(1\) during each exploitation round, whereas arms \(2\) and \(3\) are observed only during the exploration rounds. 
For the exploited arm in Figure~\ref{fig:cdf_error}(a), the ridge estimator attains a slightly smaller error than the DR estimator since the ridge estimator receives fresh observations from arm 1 in almost every round.
Nevertheless, the error of the DR estimator 
remains substantially smaller than that of the exploration-mixed estimator. 
For the unselected arms in Figure~\ref{fig:cdf_error}(b), the ridge error remains high because direct observations are available only during exploration. 
In contrast, recycling the exploration samples allows both the exploration-mixed and DR estimators to continue updating their distribution functions during exploitation, resulting in substantially smaller errors. 
Thus, the proposed DR estimator maintains accurate estimation of the
exploited arm while enabling simultaneous learning of the unselected arms.

\subsection{Dominance Score Estimation}
\label{subsec:dominance_est}
Based on the empirical distribution functions, we construct a doubly robust estimator for the dominance score. 
Define $N_{a,t}:=\sum_{s=1}^{t}\II(a_s=a)$.
First, we define the estimator as:
\begin{equation*}
    \widehat{D}_{a,t} := \frac{1}{N_{a,t}} \sum_{s=1}^{t} \mathbb{I}(a_{s}=a) \sum_{k=1}^{K} \widehat{F}_{k,t}(R_{a,s}) \mathbb{I}(R_{a,s} \in \widehat{\mathcal{G}}_{k,t}),
\end{equation*}
where $\widehat{\mathcal{G}}_{k,t}=\{x\in\RR: \widehat{F}_{k,t}(x)\le \min_{k\neq k^\prime} \widehat{F}_{k^\prime,t}(x) \}$ is the empirical partition derived from $\widehat{F}_{k,t}$. 
The following theorem establishes the concentration of this estimator around the true dominance score.

\begin{theorem}
\label{thm:D_convergence}
With probability at least $1-12\delta/t^2$, the dominance score estimator satisfies:
\begin{equation*}
    |\widehat{D}_{a,t} - D_{a}| \le \beta_{a,t} := 3\gamma_t + \sqrt{\frac{2 \log\left(\frac{Kt^2}{\delta}\right)}{N_{a,t}}},
\end{equation*}
for all $a\in[K]$.
\end{theorem}

\textit{Proof Sketch.} 
Define the following martingale difference:
\[
\epsilon_{a,s}:=\sum_{k=1}^{K}\Big(\int_{\mathcal{G}_{k}}\!\!F_{k}(y)d\mathbb{P}_{a}(y)\!-\!F_{k}(R_{a,s})\mathbb{I}(R_{a,s}\in\mathcal{G}_{k})\Big).
\] 
Observe that $\epsilon_{a,s}\in[-1,1]$. Next, define
\[
\bar{\epsilon}_{a,s}:=\sum_{k=1}^{K}\big(F_{k}(R_{a,s})-\widehat{F}{}_{k,t}(R_{a,s})\big)\mathbb{I}(R_{a,s}\in\widehat{\mathcal{G}}_{k,t}).
\]
This term represents the distribution function error component that satisfies $|\bar{\epsilon}_{a,s}|\le \max_{k\in[K]}\Vert\widehat{F}_{k,t}-F_k \Vert_{\infty}$.
Next, define the discrepancy of the partition:
\[
\check{\epsilon}_{a,s}:=\sum_{k=1}^{K}F_{k}(R_{a,s})\big(\mathbb{I}(R_{a,s}\in\mathcal{G}_{k})-\mathbb{I}(R_{a,s}\in\widehat{\Gcal}_{k,t})\big).
\]
The estimation error can be decomposed into three primary components:
\begin{equation}
\begin{aligned}
    &|\widehat{D}_{a,t}-D_{a}|\le\bigg|\frac{1}{N_{a,t}}\sum_{s=1}^{t}\mathbb{I}(a_{s}=a)\epsilon_{a,s}\bigg|\\&+\!\bigg|\frac{1}{N_{a,t}}\sum_{s=1}^{t}\mathbb{I}(a_{s}=a)\check{\epsilon}_{a,s}\bigg|+\bigg|\frac{1}{N_{a,t}}\sum_{s=1}^{t}\mathbb{I}(a_{s}=a)\bar{\epsilon}_{a,s}\bigg|.
\end{aligned}
\label{eq:error_decomp}
\end{equation}
The first term represents the concentration of the empirical dominance score, the second term accounts for the bias introduced by the empirical partition $\widehat{\mathcal{G}}_{k,t}$, and the third term captures the convergence of the distribution functions.

The first term is bounded by leveraging Azuma's inequality to obtain with probability at least $1-2\delta/(Kt^2)$
\[
\bigg|\frac{1}{N_{a,t}}\sum_{s=1}^{t}\mathbb{I}(a_{s}=a)\epsilon_{a,s}\bigg| \le \sqrt{\frac{2 \log\left(\frac{Kt^2}{\delta}\right)}{N_{a,t}}}.
\]

The second term is bounded in the following lemma.
\begin{lemma}[Partition Discrepancy Bound]
\label{lem:partition_bound}
Let $\gamma_t$ denote a high-probability bound such that $\max_k \| F_k - \widehat{F}_{k,t} \|_\infty \le \gamma_t$. 
Then,
\begin{equation*}
    \bigg|\frac{1}{N_{a,t}}\sum_{s=1}^{t}\mathbb{I}(a_{s}=a)\check{\epsilon}_{a,s}\bigg| \le 2\gamma_t.
\end{equation*}
\end{lemma}

\begin{proof}
Consider a point $y$ such that $y \in \mathcal{G}_k$ in the true partition, but $y \in \widehat{\mathcal{G}}_{k',t}$ in the empirical partition for some $k' \neq k$. By the definition of the empirical partition $\widehat{\mathcal{G}}_{k',t}$, we have $\widehat{F}_{k',t}(y) \le \widehat{F}_{k,t}(y)$ (assuming ties favor the dominant arm). 
Using the uniform convergence $\max_k \| F_k - \widehat{F}_{k,t} \|_\infty \le \gamma_t$, it follows that:
\begin{equation*}
    F_{k'}(y) - \gamma_t \le \widehat{F}_{k',t}(y) \le \widehat{F}_{k,t}(y) \le F_{k}(y) + \gamma_t.
\end{equation*}
This implies $F_{k'}(y) - F_{k}(y) \le 2\gamma_t$. Conversely, since $y \in \mathcal{G}_k$ implies $F_k(y) \le F_{k'}(y)$, we have the bound $|F_k(y) - F_{k'}(y)| \le 2\gamma_t$ for any $y$ where the true and empirical partitions disagree.

For each $a\in[K]$ and $s\in[t]$, the difference $\check{\epsilon}_{a,s}$ can be expressed by integrating over the regions where the partitions $\mathcal{G}_k$ and $\widehat{\mathcal{G}}_{k',t}$ overlap:
\begin{align*}
    |\check{\epsilon}_{a,s}|&\le\sum_{k\neq k^{\prime}}^{K}\II(R_{a,s}\in\Gcal_{k}\cap\widehat{\Gcal}_{k^{\prime},t})|F_{k}(R_{a,s})-F_{k^{\prime}}(R_{a,s})|\\
    &\le2\gamma_{t}\sum_{k=1}^{K}\sum_{k^{\prime}\in[K]\setminus\{k\}}\II(R_{a,s}\in\Gcal_{k}\cap\widehat{\Gcal}_{k^{\prime},t})\\
    &\le2\gamma_{t},
\end{align*}
where the final inequality holds because $\bigcup_{k,k'} (\mathcal{G}_k \cap \widehat{\mathcal{G}}_{k',t})$ is a partition of $\mathbb{R}$. 
This completes the proof.
\end{proof}

The last term in the error decomposition~\eqref{eq:error_decomp} uses the concentration of the doubly robust empirical measure (Corollary~\ref{cor:uniform_estimator_rate}) to obtain
\[
\bigg|\frac{1}{N_{a,t}}\sum_{s=1}^{t}\mathbb{I}(a_{s}=a)\bar{\epsilon}_{a,s}\bigg| \le \gamma_t.
\]
This completes the proof of the error bound of the dominance score.

\section{Proposed Algorithm}

In this section, we present our proposed algorithm, Doubly Robust
Dominance-Score Elimination (\texttt{DRDSE}; Algorithm~\ref{alg:DRDSE}),
which identifies the arm with the maximum dominance score.
We establish an upper bound on the sample complexity for dominant arm
identification and introduce the key technical lemmas required to prove
the convergence of the DR dominance score estimator.

\subsection{Doubly Robust Dominance-Score Elimination Algorithm}
\label{subsec:alg}
The proposed DR Dominance-Score Elimination (\texttt{DRDSE}) is a
confidence-based elimination algorithm for identifying the dominant arm.
At each round, the algorithm maintains an active set $\mathcal{A}_t$ of
candidate arms and eliminates any arm whose dominance score estimate is
sufficiently below that of another active arm, as certified by the
confidence radius $\beta_t$.
To estimate the dominance scores, the algorithm uses the doubly robust
estimator $\widehat{D}_{a,t}$ and pulls one arm per round according to
the sampling rule defined over the exploration rounds $\mathcal{E}$ and
the exploitation rounds $\mathcal{E}^{c}$.
The algorithm proceeds until a single arm remains, which is returned as
the estimated dominant arm.

\begin{algorithm}[t]
\caption{DR Dominance-Score Elimination (\texttt{DRDSE})}
\label{alg:DRDSE}
\begin{algorithmic}
\STATE \textbf{Input:} Confidence level $\delta\in(0,1)$, exploration constant $c>0$, Pseudo-action sampling probability $p\in(0,1)$, confidence radius $\beta_{a,t}:[K]\times\NN \to \RR$.
\STATE Initialize $\mathcal A_1\gets[K]$ and $t\gets1$.
\WHILE{$|\mathcal A_t|\ge2$}
    \IF{$|[t]\cap\Ecal| \le \lceil cK\log\frac{Kt^2}{\delta}\rceil$} 
        \STATE Pull all arms in $[K]$ for $K$ rounds
        \STATE Set $\Ecal \leftarrow \Ecal \cup \{t,\ldots,t+K-1\}$ and $t \leftarrow t+K$
    \ELSE
        \STATE Pull all arms in $\mathcal A_t$.
        \STATE Set $\Acal_{t+|\Acal_t|} \gets\Acal_t$ and $t\gets t+|\Acal_t|$
    \STATE Compute CDF estimators $\{\widetilde F_{k,t}\}_{k=1}^K$ and
    $\{\widehat F_{k,t}\}_{k=1}^K$.
    \STATE Construct $\{\widehat{\mathcal G}_{k,t}\}_{k=1}^K$
    and compute $\{\widehat D_{a,t}\}_{a=1}^{K}$.
    \FOR{each $b\in\mathcal A_t$}
        \IF{$\exists\,a\in\mathcal A_t$ such that
            $\widehat D_{a,t}-\widehat D_{b,t}>2\beta_t$}
            \STATE Set $\mathcal A_{t} \gets\mathcal A_{t}\setminus\{b\}$.
            \ENDIF
        \ENDFOR
    \ENDIF
    
\ENDWHILE

\STATE \textbf{Return} the remaining arm
$\widehat a^\star\in\mathcal A_t$.
\end{algorithmic}
\end{algorithm}



\begin{restatable}[Sample Complexity for Dominant Arm Identification]{theorem}{thmSampleComplexityDominantArm}
\label{thm:sample_complexity}
Let $a_\star=\arg\max_{a\in[K]}D_a$ and let $\Delta_a := D_{a_\star} - D_a > 0$ denote the sub-optimality gap for each sub-optimal arm $a \neq a_\star$. 
For any confidence $\delta \in (0, 1)$, \texttt{DRDSE} with $p=1/2$, $c=8$ and $\beta_{a,t}:=3\gamma_t+\sqrt{\frac{2\log\frac{Kt^2}{\delta}}{N_{a,t}}}$ terminates and identifies $a_\star$.
With probability at least $1 - 21\delta$, the total sample complexity is bounded by:
\begin{equation*}
64K\log\frac{16K\sqrt{K}}{\sqrt{\delta}e}+\sum_{a\neq a_{\star}} \frac{4C}{\Delta_a^2}\log\frac{C\sqrt{K} \sum_{a\neq a_{\star}} \Delta_a^{-2}} {\sqrt{\delta}e},
\end{equation*}
for some absolute constant $C>0$.
\end{restatable}
The sample complexity upper bound matches the lower bound of the sample complexity up to logarithmic factors.
In addition, the bound has the same rate as the best-arm identification literature.
This is possible due to the convergence of all $K$ reward CDFs including the unselected arms which we achieve by mixing the recycled samples and DR estimators.

\section{Numerical Experiments}
\label{sec:experiments}
We evaluate the proposed DR Dominance-Score Elimination (\texttt{DRDSE})
against the average estimator-based elimination method under the
Gaussian bandit setting and the biomedical treatment-response setting.
The Gaussian and biomedical simulations are conducted across confidence levels
$\delta\in\{0.10,\,0.05,\,0.01,\,0.005,\,0.001\}$,
while the biomedical sensitivity analysis is conducted at $\delta=0.05$.
Each method is evaluated over $500$ independent Monte Carlo runs
for each experimental condition, with sample complexity measured
by the total number of arm pulls.

\paragraph{Average Estimator-Based Elimination Method}
Let $N_{k,t} = \sum_{s=1}^{t} \II(a_s=k)$ denote the number of times
arm $k$ is pulled up to round $t$.
We define the average empirical distribution estimator as
$\hat{F}^{\mathrm{Avg}}_{k,t}(x) :=
\frac{1}{N_{k,t}} \sum_{s=1}^{t}
\mathbb{I}(a_s=k)\mathbb{I}(R_{k,s}\le x).$
The corresponding estimated partition is
$\hat{\Gcal}^{\mathrm{Avg}}_{k,t} =
\{x\in\RR: \hat{F}_{k,t}^{\mathrm{Avg}}(x)
\le \min_{k^\prime\neq k}
\hat{F}_{k^\prime,t}^{\mathrm{Avg}}(x)\}$.
The dominance score estimator is then defined as
$\hat{D}^{\mathrm{Avg}}_{a,t} :=
\frac{1}{N_{a,t}} \sum_{s=1}^{t} \mathbb{I}(a_s=a)
\sum_{k=1}^{K} \hat{F}^{\mathrm{Avg}}_{k,t}(R_{a,s})
\mathbb{I}(R_{a,s}\in\hat{\Gcal}^{\mathrm{Avg}}_{k,t})$.

\paragraph{Gaussian Bandit Setting}
We first consider $K=4$ arms with reward distributions
$R_k\sim\mathcal{N}(\mu_k,1)$, where
$(\mu_1,\mu_2,\mu_3,\mu_4)=(2.0,\,1.0,\,0.0,\,-1.0)$.
The true dominance scores are
$D_1=0.500$, $D_2=0.240$, $D_3=0.079$, and $D_4=0.017$,
so the true dominant arm is arm~$1$ with a dominance-score gap of
$\Delta=0.260$.

Figure~\ref{fig:gaussian_sample_complexity_boxplot} presents the
total number of arm pulls required to identify the dominant arm
under the Gaussian bandit setting.
While both algorithms identify the true dominant arm (arm~\(1\))
in all \(500\) runs for every value of \(\delta\), the median number of arm pulls required by \texttt{DRDSE}
is reduced by \(72.2\%\) to \(74.3\%\)
compared to the average estimator-based method.
This is consistent with the simultaneous uniform convergence of the DR estimator under the elimination-based sampling rule, whereas the average method
pulls all \(K\) arms once per iteration.

\begin{figure}[t]
    \centering
    \includegraphics[width=\columnwidth]
    {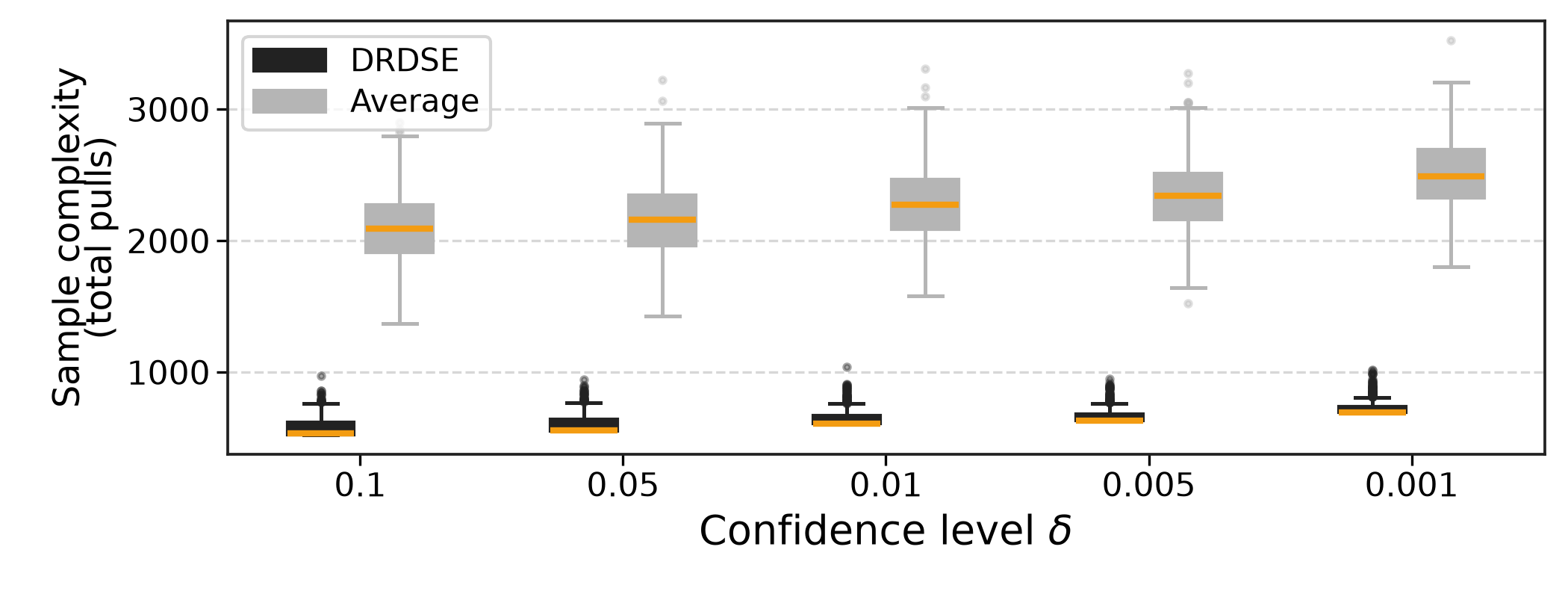}
    \caption{
        Sample complexity of \texttt{DRDSE} and the average
        estimator-based method under the Gaussian bandit setting.
    }
    \label{fig:gaussian_sample_complexity_boxplot}
\end{figure}

\paragraph{Biomedical Treatment-Response Setting}
Dose-finding studies commonly consider treatment efficacy together with
toxicity when comparing candidate treatments and dose levels
\citep{thall1998strategy,mandrekar2007adaptive,
yuan2009bayesian,cai2014bayesian}.
Motivated by the best-drug identification experiment
of \citet{wang2022best}, which uses logistic functions to model treatment
efficacy and toxicity as functions of dose, we next consider a simulated
heterogeneous treatment-response setting with $K=4$ candidate treatment
arms.

For each arm, we employ logistic dose-response models to characterize
the high-response and toxicity probabilities at a prespecified dose.
Let $d_k$ denote the normalized dose of arm $k$, and define
\[
\begin{aligned}
p_k(d)
&=
\frac{1}
{1+\exp\{-(\eta_{0,k}+\eta_{1,k}d)\}},\\
z_k(d)
&=
\frac{1}
{1+\exp\{-(\zeta_{0,k}+\zeta_{1,k}d)\}},
\end{aligned}
\]
where $p_k(d)$ and $z_k(d)$ represent the high-response and toxicity
probabilities, respectively.

In the main simulation, we set $d_k=0.5$ for all $k\in[K]$ and specify
the following synthetic logistic coefficients:
\[
\begin{aligned}
\boldsymbol{\eta}_0
&=(-1.405,\,-1.000,\,-1.000,\,-1.000),\\
\boldsymbol{\eta}_1
&=(2,\,2,\,2,\,2),\\
\boldsymbol{\zeta}_0
&=(-3.444,\,-2.697,\,-3.444,\,-3.976),\\
\boldsymbol{\zeta}_1
&=(1,\,1,\,1,\,1).
\end{aligned}
\]
At the prespecified doses, these coefficients yield high-response
probabilities $(p_1,p_2,p_3,p_4)=(0.4,\,0.5,\,0.5,\,0.5)$ and toxicity
probabilities $(0.05,\,0.10,\,0.05,\,0.03)$.
All four arms satisfy the toxicity admissibility threshold of $0.30$.
Toxicity probabilities are used only to determine arm admissibility.

To capture heterogeneity in continuous treatment responses, we model
the normalized reward of arm $k$ using a two-component Beta mixture.
The high-response component is selected with probability $p_k(d_k)$
and the low-response component with the complementary probability.
\[
R_k
\sim
\begin{cases}
\mathrm{Beta}
\!\left(\alpha_k^{\mathrm{H}},\beta_k^{\mathrm{H}}\right),
&\text{with probability }p_k(d_k),\\[2pt]
\mathrm{Beta}
\!\left(\alpha_k^{\mathrm{L}},\beta_k^{\mathrm{L}}\right),
&\text{with probability }1-p_k(d_k).
\end{cases}
\]
The parameters of the high-response and low-response components
are specified as follows:
\[
\begin{aligned}
\boldsymbol{\alpha}^{\mathrm{H}}
&=(30,\,40,\,30,\,20),\\
\boldsymbol{\beta}^{\mathrm{H}}
&=(2,\,60,\,70,\,80),\\
\boldsymbol{\alpha}^{\mathrm{L}}
&=(2,\,30,\,20,\,10),\\
\boldsymbol{\beta}^{\mathrm{L}}
&=(30,\,70,\,80,\,90).
\end{aligned}
\]

The resulting mean rewards are $0.413$, $0.350$, $0.250$, and $0.150$,
whereas the true dominance scores are
$D_1=0.320$, $D_2=0.420$, $D_3=0.151$, and $D_4=0.013$.
Although arm~$1$ maximizes the mean reward, arm~$2$ maximizes the
dominance score with a dominance-score gap of $\Delta=0.100$.

Figure~\ref{fig:biomedical_cdf} illustrates the crossing CDFs
of the treatment responses, highlighting that mean-based
identification may overlook the risk of poor individual outcomes.
The larger mean reward of arm~$1$ is driven by large outcomes from its high-response component, although a substantial proportion of its rewards remains near the lower end.
In contrast, arm~$2$ has a larger proportion of its rewards in the moderate-response range. 
The dominance score compares each candidate arm with the locally dominant arm in each region and aggregates these comparisons over the candidate arm's reward distribution, yielding $D_2>D_1$ even though arm~$1$ has the largest mean reward.

\begin{figure}[t]
    \centering
    \includegraphics[width=\columnwidth]
    {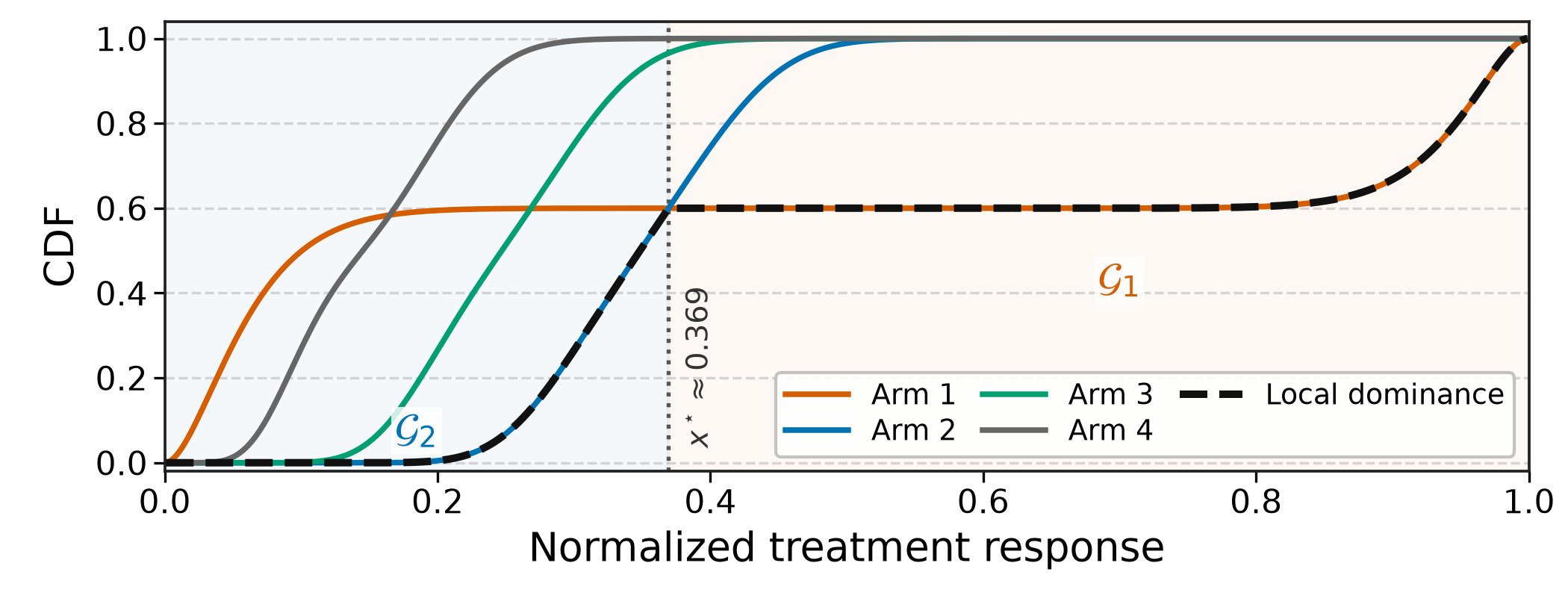}
    \caption{
        Biomedical reward CDFs and local-dominance regions.
        The black dashed curve follows the CDF of the locally dominant
        arm in each region. Arm~\(1\) has the largest mean reward,
        whereas arm~\(2\) has the largest dominance score.
    }
    \label{fig:biomedical_cdf}
\end{figure}

\begin{figure}[t]
    \centering
    \includegraphics[width=\columnwidth]
    {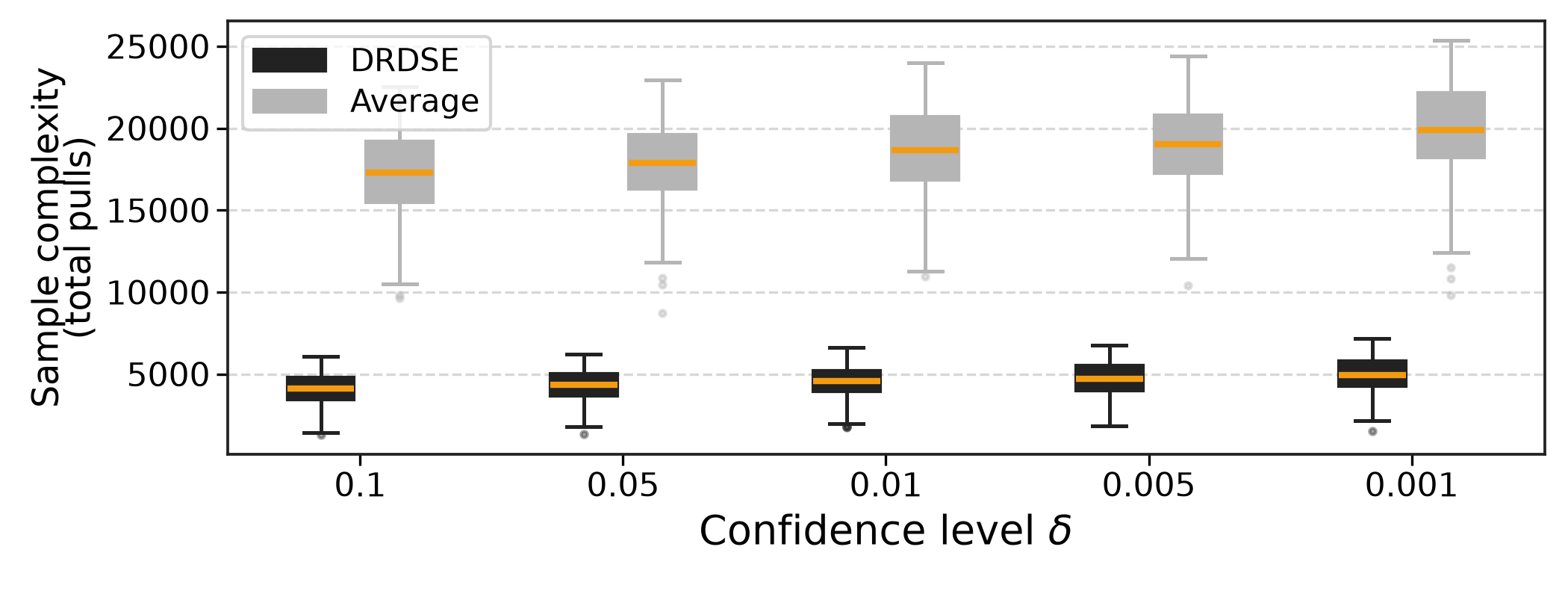}
    \caption{
        Sample complexity of \texttt{DRDSE} and the average
        estimator-based method under the biomedical treatment-response
        setting. 
        }
    \label{fig:sample_complexity_biomedical}
\end{figure}

\begin{figure}[t]
    \centering
    \includegraphics[width=\columnwidth]
    {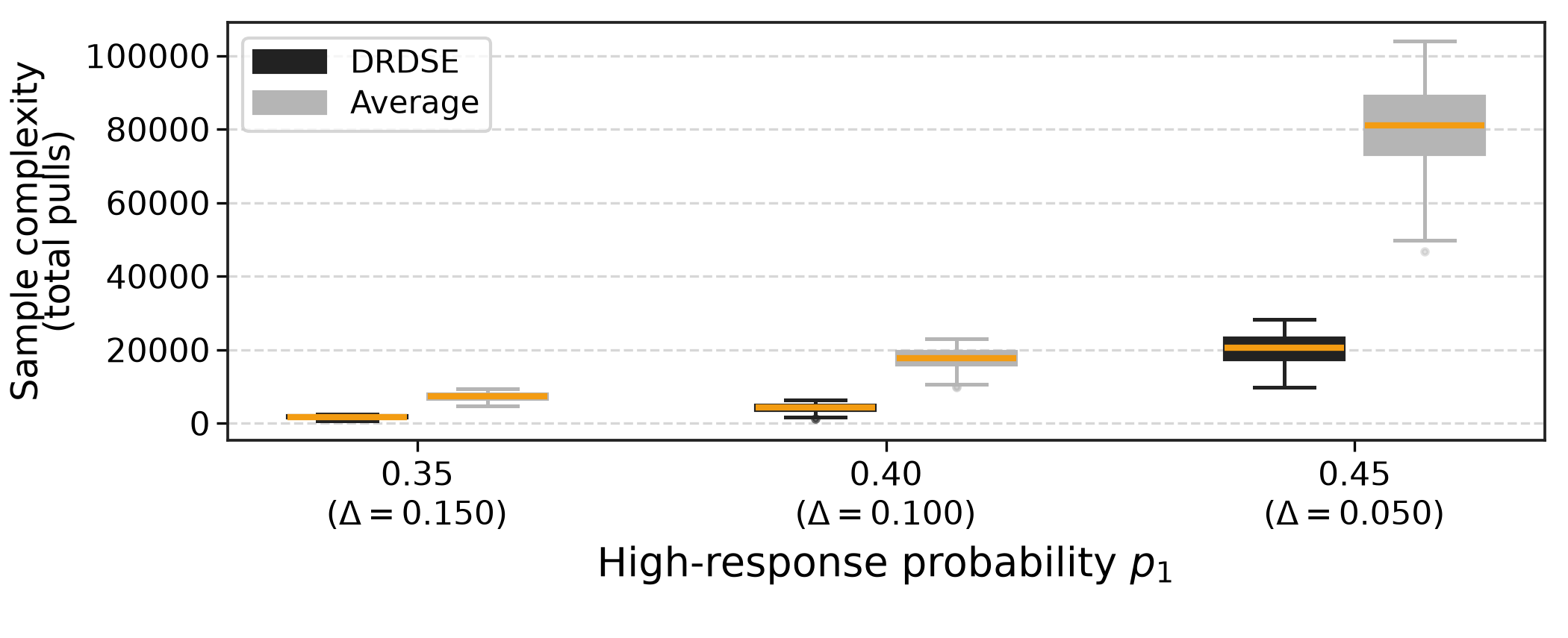}
    \caption{
        Sensitivity of sample complexity to the high-response probability
        \(p_1\) at \(\delta=0.05\) under the biomedical treatment-response
        setting. 
    }
    \label{fig:biomedical_p1_sensitivity}
\end{figure}

Figure~\ref{fig:sample_complexity_biomedical} presents the total number
of arm pulls required by \texttt{DRDSE} and the average estimator-based
method under the biomedical treatment-response setting. Across all values of \(\delta\), the median number of arm pulls
required by \texttt{DRDSE} is reduced by \(74.7\%\) to \(76.0\%\)
compared to the average estimator-based method.
These results indicate that the sample-efficiency advantage
of \texttt{DRDSE} persists in this heterogeneous,
non-Gaussian setting without compromising identification accuracy.

We further investigate the sensitivity of our results to the choice
of $p_1=0.40$. To this end, we vary only the prespecified dose of
arm~$1$, setting $d_1$ to $0.393$, $0.500$, and $0.602$, which yields
$p_1(d_1)=0.35$, $0.40$, and $0.45$, respectively.
All other doses, logistic coefficients, Beta component parameters,
and experimental parameters are held fixed.
The corresponding dominance-score gaps are approximately
$0.150$, $0.100$, and $0.050$.
Across all three settings, arm~$1$ retains the largest mean reward,
arm~$2$ retains the largest dominance score, and the toxicity
probabilities at all candidate doses remain below the admissibility
threshold.

Figure~\ref{fig:biomedical_p1_sensitivity} reports the results for
$\delta=0.05$. As the dominance-score gap decreases, both methods
require more arm pulls. The stopping time of \texttt{DRDSE} also becomes
more variable when the dominance-score gap decreases to approximately $0.05$, since some runs require additional observations to separate the two leading
arms. 
Nevertheless, the median sample complexity of \texttt{DRDSE}
is reduced by \(74.7\%\) to \(76.4\%\) compared to the average
estimator-based method.

\bibliographystyle{plainnat}
\bibliography{references}

\appendix
\onecolumn

\section{Missing Proofs}

\subsection{Proof of the Gap-Based Lower Bound}

\subsubsection{Preliminaries for the Lower Bound}
Let $\mathcal{M}$ be a class of reward distributions and
$\nu=(\mathbb{P}_1,\dots,\mathbb{P}_K)\in\mathcal{M}^K$ an instance with a unique
dominant arm. Without loss of generality, $D_1(\nu)>\max_{a\neq 1}D_a(\nu)$. Define
\[
\Delta_a:=D_1(\nu)-D_a(\nu)\quad(a\neq 1),\qquad
\Delta_{\min}:=\min_{a\neq 1}\Delta_a .
\]
An algorithm consists of a sampling rule, a stopping time $\tau$ with respect to the
filtration $(\mathcal{F}_t)_{t\ge1}$ generated by the observations, and a
recommendation $\hat a_\tau\in\mathcal{F}_\tau$. Let $N_a(\tau)$ denote the number of
pulls of arm $a$ before stopping, so that $\tau=\sum_{a=1}^K N_a(\tau)$.
The algorithm is \emph{$\delta$-correct on $\mathcal{M}$} if
$\mathbb{P}_\nu\big(\tau<\infty,\ \hat a_\tau=a^\star(\nu)\big)\ge 1-\delta$
for every $\nu\in\mathcal{M}^K$ with a unique dominant arm $a^\star(\nu)$.
The set of alternatives to $\nu$ is
\[
\mathrm{Alt}(\nu):=\Big\{\nu'\in\mathcal{M}^K:\ \exists\, b\neq 1,\
D_b(\nu')>\max_{c\neq b}D_c(\nu')\Big\}.
\]
Let $\mathrm{kl}(x,y):=x\log\frac{x}{y}+(1-x)\log\frac{1-x}{1-y}$ denote the Bernoulli
relative entropy, and recall that
$\mathrm{kl}(\delta,1-\delta)\ge\log\frac{1}{2.4\delta}$ for all $\delta\in(0,1)$
\citep{kaufmann2016complexity}.

\begin{lemma}[Transportation lemma; {\citealp[Lemma~1]{kaufmann2016complexity}}]
\label{lem:transport}
Let $\nu,\nu'$ be two instances such that
$\mathrm{KL}(\mathbb{P}_a\|\mathbb{P}'_a)<\infty$ for all $a\in[K]$, and suppose
$\mathbb{E}_\nu[\tau]<\infty$. Then for every event $E\in\mathcal{F}_\tau$,
\[
\sum_{a=1}^K \mathbb{E}_\nu[N_a(\tau)]\,\mathrm{KL}(\mathbb{P}_a\|\mathbb{P}'_a)
\;\ge\; \mathrm{kl}\big(\mathbb{P}_\nu(E),\mathbb{P}_{\nu'}(E)\big).
\]
\end{lemma}

\begin{theorem}[Instance-dependent lower bound]
\label{thm:lb_general}
For every $\delta\in(0,\tfrac12)$ and every $\delta$-correct algorithm on $\mathcal{M}$,
\[
\mathbb{E}_\nu[\tau]\;\ge\;T^\star(\nu)\,\log\frac{1}{2.4\delta},
\qquad
T^\star(\nu)^{-1}:=\sup_{w\in\Sigma_K}\ \inf_{\nu'\in\mathrm{Alt}(\nu)}
\sum_{a=1}^K w_a\,\mathrm{KL}(\mathbb{P}_a\|\mathbb{P}'_a),
\]
where $\Sigma_K$ denotes the probability simplex. In particular, every
$\nu'\in\mathrm{Alt}(\nu)$ satisfies
\begin{equation}\label{eq:per_alt}
\sum_{a=1}^K \mathbb{E}_\nu[N_a(\tau)]\,\mathrm{KL}(\mathbb{P}_a\|\mathbb{P}'_a)
\;\ge\;\log\frac{1}{2.4\delta}.
\end{equation}
\end{theorem}

\begin{proof}
If $\mathbb{E}_\nu[\tau]=\infty$ there is nothing to prove. Fix
$\nu'\in\mathrm{Alt}(\nu)$ and let $E:=\{\hat a_\tau=1\}$. By $\delta$-correctness,
$\mathbb{P}_\nu(E)\ge1-\delta$; since arm $1$ is not the dominant arm of $\nu'$,
$\mathbb{P}_{\nu'}(E)\le\delta$. As $\mathrm{kl}(x,y)$ is nondecreasing in $x$ and
nonincreasing in $y$ on $\{x\ge y\}$, Lemma~\ref{lem:transport} yields
\eqref{eq:per_alt} with right-hand side
$\mathrm{kl}(1-\delta,\delta)=\mathrm{kl}(\delta,1-\delta)\ge\log\frac{1}{2.4\delta}$.
Setting $w_a:=\mathbb{E}_\nu[N_a(\tau)]/\mathbb{E}_\nu[\tau]\in\Sigma_K$,
\[
\log\frac{1}{2.4\delta}
\le \mathbb{E}_\nu[\tau]\inf_{\nu'\in\mathrm{Alt}(\nu)}\sum_{a=1}^K
w_a\,\mathrm{KL}(\mathbb{P}_a\|\mathbb{P}'_a)
\le \mathbb{E}_\nu[\tau]\,T^\star(\nu)^{-1}. \qedhere
\]
\end{proof}

\subsubsection{A Hard Instance Family}
For $p\in(0,1)$, define the distribution
\[
\mathbb{P}^{(p)}:=(1-p)\,\mathrm{Unif}[0,1]+p\,\mathrm{Unif}[1,2],
\]
and let $\mathcal{F}:=\{\mathbb{P}^{(p)}:p\in(0,1)\}$.

\begin{lemma}\label{lem:family}
Let $\nu=(\mathbb{P}^{(p_1)},\dots,\mathbb{P}^{(p_K)})$ and $p^\star:=\max_k p_k$. Then
\[
D_a(\nu)=\tfrac12\big(1-p^\star+p_a\big)\quad(a\in[K]),
\qquad
\mathrm{KL}\big(\mathbb{P}^{(p)}\,\|\,\mathbb{P}^{(q)}\big)=\mathrm{kl}(p,q).
\]
In particular, if $p_1>\max_{a\neq1}p_a$, then $D_1(\nu)=\tfrac12$ and
$\Delta_a=\tfrac12(p_1-p_a)$.
\end{lemma}

\begin{proof}
The CDF of $\mathbb{P}^{(p)}$ is
\[
F_p(r)=
\begin{cases}
(1-p)\,r, & r\in[0,1],\\
(1-p)+p(r-1), & r\in[1,2].
\end{cases}
\]
Since $\partial_pF_p(r)=-r\le0$ on $[0,1]$ and $\partial_pF_p(r)=r-2\le0$ on $[1,2]$,
we have $\min_kF_{p_k}=F_{p^\star}$ pointwise. Hence
\[
D_a(\nu)=\int F_{p^\star}\,d\mathbb{P}^{(p_a)}
=(1-p_a)\int_0^1(1-p^\star)\,r\,dr
+p_a\int_1^2\big[(1-p^\star)+p^\star(r-1)\big]\,dr
=\frac{1-p^\star}{2}+\frac{p_a}{2}.
\]
The density of $\mathbb{P}^{(p)}$ equals $1-p$ on $[0,1]$ and $p$ on $[1,2]$, so
$\mathrm{KL}(\mathbb{P}^{(p)}\|\mathbb{P}^{(q)})
=(1-p)\log\frac{1-p}{1-q}+p\log\frac{p}{q}=\mathrm{kl}(p,q)$.
\end{proof}

\thmSampleComplexityLowerBound*
\begin{proof}
We use that $q(1-q)\ge\alpha(1-\alpha)$ for all
$q\in[\alpha,1-\alpha]$, and that, because the probability measure distance $\mathrm{KL}\le\chi^2$,
\begin{equation}\label{eq:kl_chi2}
\mathrm{kl}(p,q)\;\le\;\frac{(p-q)^2}{q(1-q)}\qquad(p,q\in(0,1)).
\end{equation}

Fix $a\neq1$ and $\varepsilon\in(0,\alpha)$. Let
$\nu^{\varepsilon}$ coincide with $\nu$ except that arm $a$ has distribution
$\mathbb{P}^{(p_1+\varepsilon)}$; note $p_1+\varepsilon<1$ since
$p_1\le1-\alpha$. By Lemma~\ref{lem:family}, arm $a$ is the unique dominant arm of
$\nu^\varepsilon$, so $\nu^\varepsilon\in\mathrm{Alt}(\nu)$. Since only arm $a$ is modified, 
\[
\mathbb{E}_\nu[N_a(\tau)]\;\mathrm{kl}(p_a,p_1+\varepsilon)\;\ge\;\log\frac{1}{2.4\delta}.
\]
Letting $\varepsilon\downarrow0$ and using continuity of $\mathrm{kl}$ gives
$\mathbb{E}_\nu[N_a(\tau)]\ge\log\frac{1}{2.4\delta}\big/\mathrm{kl}(p_a,p_1)$.
By \eqref{eq:kl_chi2} with $q=p_1\in[\alpha,1-\alpha]$,
\[
\mathrm{kl}(p_a,p_1)\;\le\;\frac{(p_1-p_a)^2}{\alpha(1-\alpha)}
\;=\;\frac{4\,\Delta_a^2}{\alpha(1-\alpha)},
\]
which proves the bound for $a\neq1$.

Summing the bounds over $a\in[K]$ and using
$\mathbb{E}_\nu[\tau]=\sum_{a}\mathbb{E}_\nu[N_a(\tau)]$ yields the boxed inequality.
For realizability, set $p_1=1-\alpha$ and $p_a=1-\alpha-2\Delta_a$. Then
$p_a\ge\alpha$ if and only if $\Delta_a\le\tfrac12-\alpha$, and
$\Delta_a=\tfrac12(p_1-p_a)$ by Lemma~\ref{lem:family}.
\end{proof}

\subsection{Properties of the Mixed Gram matrix}
\lemMixedGramMatrix*

\begin{proof}
Recall that 
\[
\tilde{V}_{t}:=\sum_{s\in[t]\cap\Ecal}w_{s}\eb_{a_{s}}\eb_{a_{s}}^{\top}+\sum_{s\in[t]\setminus\Ecal}w_{s}\tilde{X}_{a_{s},s}\tilde{X}_{a_{s},s}^{\top}.
\]
Let $\EE_{s}$ denote the conditional expectation given weights $w_{1},\ldots,w_{s+1}$,
actions $a_{1},\ldots,a_{s+1}$ and random variables $m_{1,1},\ldots,m_{K,s}$.
Note that $\|\tilde{X}_{a_{s},s}\tilde{X}_{a_{s},s}^{\top}\|_{2}\le4$
and $\EE_{s}[w_{s}\tilde{X}_{a_{s},s}\tilde{X}_{a_{s},s}^{\top}|]=w_{s}\Ib_{K}/K+w_{s}\eb_{a_{s}}\eb_{a_{s}}^{\top}$.
By the matrix Bernstein bound (Lemma~\ref{lem:matrix_hoeffding}), with probability at least $1-2\delta/t^{2}$,
for $v\in(0,\frac{1}{8}],$
\[
\Big\Vert\sum_{s\in[t]\setminus\Ecal}w_{s}\tilde{X}_{a_{s},s}\tilde{X}_{a_{s},s}^{\top}-\sum_{s\in[t]\setminus\Ecal}w_{s}(\Ib_{K}/K+\eb_{a_{s}}\eb_{a_{s}}^{\top})\Big\Vert_{2}\le v\Big\Vert\sum_{s\in[t]\setminus\Ecal}w_{s}^{2}\EE_{s}[\tilde{X}_{a_{s},s}\tilde{X}_{a_{s},s}^{\top}\tilde{X}_{a_{s},s}\tilde{X}_{a_{s},s}^{\top}]\Big\Vert_{2}+\frac{1}{v}\log\frac{Kt^{2}}{\delta}.
\]
Note that
\begin{align*}
\Big\Vert\sum_{s\in[t]\setminus\Ecal}w_{s}^{2}\EE_{s}[\tilde{X}_{a_{s},s}\tilde{X}_{a_{s},s}^{\top}\tilde{X}_{a_{s},s}\tilde{X}_{a_{s},s}^{\top}]\Big\Vert_{2}\le & 4\Big\Vert\sum_{s\in[t]\setminus\Ecal}w_{s}^{2}\EE_{s}[\tilde{X}_{a_{s},s}\tilde{X}_{a_{s},s}^{\top}]\Big\Vert_{2}\\
\le & 16\Big\Vert\sum_{s\in[t]\setminus\Ecal}\Ib_{K}/K+\eb_{a_{s}}\eb_{a_{s}}^{\top}\Big\Vert_{2}\\
\le & 32|[t]\setminus\Ecal|.
\end{align*}
Thus, setting $v=\sqrt{\frac{\log\frac{Kt^{2}}{\delta}}{32|[t]\setminus\Ecal|}}$,
gives
\[
\Big\Vert\sum_{s\in[t]\setminus\Ecal}w_{s}\tilde{X}_{a_{s},s}\tilde{X}_{a_{s},s}^{\top}-\sum_{s\in[t]\setminus\Ecal}w_{s}(\Ib_{K}/K+\eb_{a_{s}}\eb_{a_{s}}^{\top})\Big\Vert_{2}\le8\sqrt{2|[t]\setminus\Ecal|\log\frac{Kt^{2}}{\delta}}
\]
Thus, with probability at least $1-2\delta/t^{2}$,
\[
\Big\Vert\tilde{V}_{t}-\sum_{s=1}^{t}w_{s}\eb_{a_{s}}\eb_{a_{s}}^{\top}-\sum_{s\in[t]\setminus\Ecal}w_{s}\frac{\Ib_{K}}{K}\Big\Vert_{2}\le8\sqrt{2t\log\frac{Kt^{2}}{\delta}}
\]
Since $\|\sum_{k=1}^{K}\eb_{k}\eb_{k}^{\top}+\Ib_{K}\|_{2}=2$, by the importance weight bound (Lemma~\ref{lem:importance_weight}), with probability at least $1-4\delta/t^{2}$
\begin{align*}
 & \Big\Vert\sum_{s=1}^{t}w_{s}\eb_{a_{s}}\eb_{a_{s}}^{\top}+\sum_{s\in[t]\setminus\Ecal}w_{s}\frac{\Ib_{K}}{K}-(|[t]\cap\Psi|+|[t]\setminus\Ecal\cap\Psi)\Ib_{K}\Big\Vert_{2}\\
 & \le2|[t]\setminus\Psi|+8\sqrt{2Kt\log\frac{Kt^{2}}{\delta}}.
\end{align*}
Thus,
\[
\Big\Vert\tilde{V}_{t}-(|[t]\cap\Psi|+|[t]\setminus\Ecal\cap\Psi)\Ib_{K}\Big\Vert_{2}\le8\sqrt{2t\log\frac{Kt^{2}}{\delta}}+|[t]\setminus\Psi|+8\sqrt{2Kt\log\frac{Kt^{2}}{\delta}}.
\]
By construction of $\Psi$, applying Lemma~\ref{lem:matrix_hoeffding} with $K=1$, with probability at least $1-2\delta/t^{2}$, for $v\in(0,1/2]$
\[
|[t]\cap\Psi|\ge\sum_{s=1}^{t}\PP(s\in\Psi)-v\sum_{s=1}^{t}\PP(s\in\Psi)-\frac{1}{v}\log\frac{Kt^{2}}{\delta}.
\]
Note that
\[
\sum_{s=1}^{t}\PP(s\in\Psi)=t-\sum_{s=1}^{t}\frac{1}{s+1}\ge t-\log(t+1).
\]
Thus, setting $v=1/2$, 
\[
|[t]\cap\Psi|\ge\frac{t}{2}-\frac{\log(t+1)}{2}-2\log\frac{Kt^{2}}{\delta}\ge\frac{t}{3},
\]
where the last inequality holds because $t\ge15\log\frac{Kt^{2}}{\delta}$.
Again, by Lemma~\ref{lem:matrix_hoeffding}, with probability at least
$1-2\delta/t^{2}$ 
\begin{align*}
|[t]\setminus\Psi|\le & \sum_{s=1}^{t}\PP(s\notin\Psi)+\frac{1}{2}\sum_{s=1}^{t}\PP(s\notin\Psi)+2\log\frac{Kt^{2}}{\delta}\\
\le & \frac{3}{2}\log(t+1)+2\log\frac{Kt^{2}}{\delta}\\
\le & \frac{7}{2}\log\frac{Kt^{2}}{\delta}.
\end{align*}
Because $t/15\ge\log\frac{Kt^{2}}{\delta}$, 
\begin{align*}
\Big\Vert\tilde{V}_{t}-(|[t]\cap\Psi|+|[t]\setminus\Ecal\cap\Psi)\Ib_{K}\Big\Vert_{2}\le & 8\sqrt{2t\log\frac{Kt^{2}}{\delta}}+\frac{7}{2}\log\frac{Kt^{2}}{\delta}+8\sqrt{2Kt\log\frac{Kt^{2}}{\delta}}\\
\le & \sqrt{2t\log\frac{Kt^{2}}{\delta}}\Big(8+\frac{7}{30\sqrt{2}}+8\sqrt{K}\Big)\\
\le & \sqrt{2t\log\frac{Kt^{2}}{\delta}}\Big(9+8\sqrt{K}\Big)
\end{align*}
Since $t$ satisfies $t\ge 8K\log\frac{Kt^2}{\delta}$ applying Lemma~\ref{lem:logt} gives $t\ge72\log\frac{Kt^{2}}{\delta}(9+8\sqrt{K})^{2}$.
It follows that
\[
\sqrt{2t\log\frac{Kt^{2}}{\delta}}\Big(9+8\sqrt{K}\Big)\le\frac{t}{6}.
\]
Thus,
\[
\Big\Vert\tilde{V}_{t}-(|[t]\cap\Psi|+|[t]\setminus\Ecal\cap\Psi)\Ib_{K}\Big\Vert_{2}\le\frac{t}{6}.
\]
It follows that
\[
\frac{t}{6}\Ib_{K}\preceq\tilde{V}_{t}\preceq\frac{13t}{6}\Ib_{K},
\]
which concludes the proof.
\end{proof}

\subsection{An Error Bound for the Mixed estimator}

\lemMixedEstimator*

\begin{proof}
Let $\eta_{s}(x):=\II(R_{a_{s},s}\le x)-\PP(R_{a_{s},s}\le x)$ and
\[
\tilde{\eta}_{s}(x):=\tilde{R}_{s}(x)-F(x)^{\top}\tilde{X}_{s}=\frac{1}{K}\sum_{k=1}^{K}m_{k,s}\eta_{\tau_{s}+k-1}(x)+\eta_{s}(x),
\]
where recall that $\tau_{s}:=\min\{\argmin{\tau\in\Ecal\cap[s]}\sum_{\nu\in[s-1]\setminus\Ecal}\II(\tau=\tau_{\nu})\}$
denote the index of the minimally reused round. 
For each $k\in[K]$, let $\eb_k$ denote the Euclidean basis in $\RR^K$.
Then,
\[
\eb_{k}^{\top}(\tilde{\Fb}_{t}(x)-\Fb(x))=\eb_{k}^{\top}\tilde{V}_{t}^{-1}\Big(\sum_{s\in[t]\cap\Ecal}w_{s}\eb_{a_{s}}\eta_{s}(x)+\sum_{s\in[t]\setminus\Ecal}w_{s}\tilde{\eta}_{s}(x)\tilde{X}_{s}\Big).
\]
Recall that $\tilde{V}_{t}:=\sum_{s\in[t]\cap\Ecal}w_{s}\eb_{a_{s}}\eb_{a_{s}}^{\top}+\sum_{s\in[t]\setminus\Ecal}w_{s}\tilde{X}_{s}\tilde{X}_{s}^{\top}$.
Observe that 
\[
\sum_{s\in[t]\setminus\Ecal}w_{s}\tilde{\eta}_{s}(x)\tilde{X}_{s}=\frac{1}{K}\sum_{s\in[t]\setminus\Ecal}\sum_{k=1}^{K}m_{k,s}w_{s}\eta_{\tau_{s}+k-1}(x)\tilde{X}_{s}+\sum_{s\in[t]\setminus\Ecal}w_{s}\eta_{s}(x)\tilde{X}_{s}
\]
Then, by the triangle inequality, 
\begin{align*}
\big|\eb_{k}^{\top}(\tilde{\Fb}_{t}(x)-\Fb(x))\big|= & \big|\sum_{s\in[t]\cap\Ecal}w_{s}\eta_{s}(x)\eb_{k}^{\top}\tilde{V}_{t}^{-1}\eb_{a_{s}}+\sum_{s\in[t]\cap\Ecal^{c}}w_{s}\tilde{\eta}_{s}(x)\eb_{k}^{\top}\tilde{V}_{t}^{-1}\tilde{X}_{s}\big|\\
\le & |\sum_{s\in[t]\cap\Ecal}w_{s}\eta_{s}(x)\eb_{k}^{\top}\tilde{V}_{t}^{-1}\eb_{a_{s}}+\sum_{s\in[t]\setminus\Ecal}w_{s}\eta_{s}(x)\eb_{k}^{\top}\tilde{V}_{t}^{-1}\tilde{X}_{s}|\\
 & \quad+\big|\frac{1}{K}\sum_{s\in[t]\setminus\Ecal}\sum_{k=1}^{K}m_{k,s}w_{s}\eta_{\tau_{s}+k-1}(x)\eb_{k}^{\top}\tilde{V}_{t}^{-1}\tilde{X}_{s}\big|.
\end{align*}
For the first term, by the Cauchy-Schwarz inequality,
\begin{align*}
 & |\sum_{s\in[t]\cap\Ecal}w_{s}\eta_{s}(x)\eb_{k}^{\top}\tilde{V}_{t}^{-1}\eb_{a_{s}}+\sum_{s\in[t]\setminus\Ecal}w_{s}\eta_{s}(x)\eb_{k}^{\top}\tilde{V}_{t}^{-1}\tilde{X}_{s}|\\
 & \le\Vert\eb_{k}\Vert_{\tilde{V}_{t}^{-1}}\big\Vert\sum_{s\in[t]\cap\Ecal}w_{s}\eta_{s}(x)\eb_{a_{s}}+\sum_{s\in[t]\setminus\Ecal}w_{s}\eta_{s}(x)\tilde{X}_{s}\big\Vert_{\tilde{V}_{t}^{-1}}.
\end{align*}
By Lemma~\ref{lem:mixed_gram_matrix_spectrum}, with probability at least $1-10\delta/t^2$ we obtain $\tilde{V}_{t}\succeq(t/6)\Ib_{K}$ and 
\begin{align*}
 & |\sum_{s\in[t]\cap\Ecal}w_{s}\eta_{s}(x)\eb_{k}^{\top}\tilde{V}_{t}^{-1}\eb_{a_{s}}+\sum_{s\in[t]\setminus\Ecal}w_{s}\eta_{s}(x)\eb_{k}^{\top}\tilde{V}_{t}^{-1}\tilde{X}_{s}|\\
 & \le\frac{6}{t}\Vert\eb_{k}\Vert_{2}\big\Vert\sum_{s\in[t]\cap\Ecal}w_{s}\eta_{s}(x)\eb_{a_{s}}+\sum_{s\in[t]\setminus\Ecal}w_{s}\eta_{s}(x)\tilde{X}_{s}\big\Vert_{2}\\
 & =\frac{6}{t}\big\Vert\sum_{s\in[t]\cap\Ecal}w_{s}\eta_{s}(x)\eb_{a_{s}}+\sum_{s\in[t]\setminus\Ecal}w_{s}\eta_{s}(x)\tilde{X}_{s}\big\Vert_{2}.
\end{align*}
Note that $\sqrt{w_{s}}\eta_{s}(x)\in[-\sqrt{2},\sqrt{2}]$ and $\|\tilde{X}_{s}\|_{2}^{2}\le8$.
Thus, by the dimension-free martingale bound (Lemma~\ref{lem:dimension_free_martingale}), for each
$x_{11},\ldots,x_{tK}\in\RR$, with probability at least $1-2\delta/(Kt^{2})$,
\[
\sqrt{\sum_{j=1}^{K}\Big(\sum_{s\in[t]\cap\Ecal}w_{s}\eta_{s}(x_{sj})\II(j=a_{s})+\sum_{s\in[t]\setminus\Ecal}w_{s}\eta_{s}(x_{sj})\tilde{X}_{sj}}\Big)^{2}\le4\sqrt{2\cdot8t\log\frac{Kt^{2}}{\delta}}.
\]
By the uniform bound (Lemma~\ref{lem:vector_pointwise_to_uniform_vcdf}),
\[
\sup_{x\in\RR}\big\Vert\sum_{s\in[t]\cap\Ecal}w_{s}\eta_{s}(x)\eb_{a_{s}}+\sum_{s\in[t]\setminus\Ecal}w_{s}\eta_{s}(x)\tilde{X}_{s}\big\Vert_{2}\le4\sqrt{2\cdot8t\log\frac{2Kt^{2}}{\delta}}+\frac{\sqrt{8t}}{2}\le17\sqrt{t\log\frac{2Kt^{2}}{\delta}}.
\]
Thus,
\[
\sup_{x\in\RR}|\sum_{s\in[t]\cap\Ecal}w_{s}\eta_{s}(x)\eb_{k}^{\top}\tilde{V}_{t}^{-1}\eb_{a_{s}}+\sum_{s\in[t]\setminus\Ecal}w_{s}\eta_{s}(x)\eb_{k}^{\top}\tilde{V}_{t}^{-1}\tilde{X}_{s}|\le102\sqrt{\frac{3}{t}\log\frac{Kt^{2}}{\delta}}.
\]
For the second term, by the Cauchy-Schwarz inequality and by the
fact that $\tilde{V}_{t}\succeq(t/6)\Ib_{K}$, we obtain
\begin{align*}
 & \big|\frac{1}{K}\sum_{s\in[t]\setminus\Ecal}\sum_{k=1}^{K}m_{k,s}w_{s}\eta_{\tau_{s}+k-1}(x)\eb_{k}^{\top}\tilde{V}_{t}^{-1}\tilde{X}_{s}\big|\\
 & \le\frac{6}{tK}\Big\Vert\sum_{s\in[t]\setminus\Ecal}\sum_{k=1}^{K}m_{k,s}w_{s}\eta_{\tau_{s}+k-1}(x)\tilde{X}_{s}\Big\Vert_{2}.
\end{align*}
For the last term, for each $u\in\Ecal$, let $\Tcal_{u}:=\{s\in\Ecal^{c}:\tau_{s}=u\}$.
Then 
\[
\sum_{s\in[t]\cap\Ecal^{c}}\sum_{k=1}^{K}m_{k,s}w_{s}\eta_{\tau_{s}+k-1}(x)\tilde{X}_{s}=\sum_{u\in\Ecal\cap[t]}\sum_{k=1}^{K}\eta_{u+k-1}(x)\sum_{s\in\Tcal_{u}\cap[t]}m_{k,s}w_{s}\tilde{X}_{s}.
\]
For each $u\in\Ecal$, the explored sample is used not more than $Kt/|[t]\cap\Ecal|$
times. Thus $|\Tcal_{u}\cap[t]|\le Kt/|[t]\cap\Ecal|$. Note that
$\sqrt{w_{s}}\eta_{s}(x)\in[-\sqrt{2},\sqrt{2}]$ and $\Vert\sum_{s\in\Tcal_{u}\cap[t]}m_{k,s}w_{s}\tilde{X}_{s}\Vert_{2}^{2}\le144K^{2}t^{2}/|[t]\cap\Ecal|^{2}$.
By the dimension-free martingale bound (Lemma~\ref{lem:dimension_free_martingale}), with probability
at least $1-2\delta/(Kt^{2})$,
\[
\sqrt{\sum_{j=1}^{K}\big(\sum_{u\in\Ecal\cap[t]}\eta_{u}(x_{uj})\sum_{s\in\Tcal_{u}\cap[t]}m_{k,s}w_{s}\tilde{X}_{sj}}\big)^{2}\le48\sqrt{\frac{K^{2}t^{2}}{|[t]\cap\Ecal|}\log\frac{Kt^{2}}{\delta}}.
\]
By the uniform bound (Lemma~\ref{lem:vector_pointwise_to_uniform_vcdf}),
\begin{align*}
\sup_{x\in\RR}\Big\Vert\sum_{u\in\Ecal\cap[t]}\sum_{k=1}^{K}\eta_{u+k-1}(x)\sum_{s\in\Tcal_{u}\cap[t]}m_{k,s}w_{s}\tilde{X}_{s}\Big\Vert_{2}\le & 48\sqrt{\frac{K^{2}t^{2}}{|[t]\cap\Ecal|}\log\frac{2Kt^{2}}{\delta}}+\frac{1}{2}\sqrt{\frac{K^{2}t^{2}}{|[t]\cap\Ecal|}}\\
\le & 49\sqrt{\frac{K^{2}t^{2}}{|[t]\cap\Ecal|}\log\frac{2Kt^{2}}{\delta}}
\end{align*}
Thus, 
\[
\sup_{x\in\RR}\big|\frac{1}{K}\sum_{s\in[t]\setminus\Ecal}\sum_{k=1}^{K}m_{k,s}w_{s}\eta_{\tau_{s}+k-1}(x)\eb_{k}^{\top}\tilde{V}_{t}^{-1}\tilde{X}_{s}\big|\le6\cdot49\sqrt{\frac{1}{|[t]\cap\Ecal|}\log\frac{2Kt^{2}}{\delta}}.
\]
which concludes
\[
\big|\eb_{k}^{\top}(\tilde{\Fb}_{t}(x)-\Fb(x))\big|\le102\sqrt{\frac{3}{t}\log\frac{Kt^{2}}{\delta}}+6\cdot49\sqrt{\frac{1}{|[t]\cap\Ecal|}\log\frac{2Kt^{2}}{\delta}}.
\]
\end{proof}

\subsection{An Error bound for the Doubly Robust Estimator}
\thmDREstimatorErrorBound*

\begin{proof}
Let us fix $t\ge T_{e}$ throughout the proof. For $s\in[t]$, let
$\widehat{\eta}_{k,s}(x):=\widehat{R}_{k,s}(x)-F_{k}(x).$ For each
$s\in[t]$ and $n\in[\rho_{s}]$, let $\tilde{\Fcal}_{n,s}$ denote
the sigma algebra generated by the random variables $\tilde{a}_{s}(1),\ldots,\tilde{a}_{s}(n)$.
For the stopping time $\tau_{s}:=\inf\{n\in[\rho_{s}]:\tilde{a}_{s}(n)=a_{s}\}$,
we define the sigma algebra $\Fcal_{\tau_{s}}^{-}:=\sigma\{A\cap\{\tau_{s}=n\}:A\in\tilde{\Fcal}_{n-1,s},\forall n\ge1\}$.
Then it follows that $\PP(\tilde{a}_{s}(\tau_{s})=a_{s}|\Fcal_{\tau_{s}}^{-})=1/2$
and $\PP(\tilde{a}_{s}(\tau_{s})=k|\Fcal_{\tau_{s}}^{-})=1/(2K-2)$
for $k\neq a_{s}$. By definition of the DR estimator and the pseudo-reward
$\widehat{R}_{k,s}(x)$, and for each $k\in[K]$, 
\begin{align*}
|\widehat{F}_{k,t}(x)-F_{k}(x)|= & \frac{1}{|[t]\cap\Psi|+\sum_{s\in[t]\setminus\Psi}\II(a_{s}=k)}\Big|\sum_{s=1}^{t}\widehat{\eta}_{k,s}(x)\Big|\\
= & \frac{1}{|[t]\cap\Psi|+\sum_{s\in[t]\setminus\Psi}\II(a_{s}=k)}\Big|\sum_{s=1}^{t}w_{s}\eta_{k,s}(x)+\big(\sum_{s\in\Psi\cap[t]}\big(1-\frac{\II(\tilde{a}_{s}(\tau_{s})=k)}{\PP(\tilde{a}_{s}(\tau_{s})=k|\Fcal_{\tau_{s}}^{-})}\big)\eb_{k}\eb_{k}^{\top}\big)\big(\tilde{\Fb}_{t}(x)-\Fb(x)\big)\Big|.
\end{align*}
As in Lemma~\ref{lem:mixed_gram_matrix_spectrum}, applying the Bernstein bound (Lemma~\ref{lem:matrix_hoeffding}), with probability at least $1-\delta/t^2$,
\[
|[t]\cap\Psi|+\sum_{s\in[t]\setminus\Psi}\II(a_{s}=k)\ge|[t]\cap\Psi|\ge\frac{t}{3}.
\]
Thus,
\[
|\widehat{F}_{k,t}(x)-F_{k}(x)|\le\frac{3}{t}\Big|\sum_{s=1}^{t}w_{s}\eta_{k,s}(x)+\big(\sum_{s\in\Psi\cap[t]}\big(1-\frac{\II(\tilde{a}_{s}(\tau_{s})=k)}{\PP(\tilde{a}_{s}(\tau_{s})=k|\Fcal_{\tau_{s}}^{-})}\big)\eb_{k}\eb_{k}^{\top}\big)\big(\tilde{\Fb}_{t}(x)-\Fb(x)\big)\Big|
\]
Let us write $U_{k,s}:=1-\frac{\II(\tilde{a}_{s}(\tau_{s})=k)}{\PP(\tilde{a}_{s}(\tau_{s})=k|\Fcal_{\tau_{s}}^{-})}$.
By the triangle inequality, 
\begin{align*}
 & |\widehat{F}_{k,t}(x)-F_{k}(x)|\\
 & \le\frac{3}{t}\Big|\sum_{s=1}^{t}w_{s}\eta_{k,s}(x)\Big|+\frac{3}{t}\Big|\eb_{k}^{\top}\big(\sum_{s\in\Psi\cap[t]}U_{k,s}\eb_{k}\eb_{k}^{\top}\big)\big(\tilde{\Fb}_{t}(x)-\Fb(x)\big)\Big|.
\end{align*}
For the first term, since $w_{s}\le2$, by the Hoeffding inequality, with probability at least $1-2\delta/(Kt^{2})$, 
\[
\frac{1}{t}\Big|\sum_{s=1}^{t}w_{s}\eta_{k,s}(x)\Big|\le\sqrt{\frac{8}{t}\log\frac{Kt^{2}}{\delta}}.
\]
From pointwise to uniform lemma (Lemma~\ref{lem:pointwise_to_uniform_cdf}) setting $M=t$ gives 
\[
\frac{1}{t}\sup_{x\in\RR}\Big|\sum_{s=1}^{t}w_{s}\eta_{k,s}(x)\Big|\le\sqrt{\frac{8}{t}\log\frac{Kt^{3}}{\delta}}+\frac{2}{t}\le3\sqrt{\frac{1}{t}\log\frac{Kt^{3}}{\delta}},
\]
where the last inequality holds since $t\ge4(3-\sqrt{8})^{-2}\log\frac{Kt^{3}}{\delta}$.
For the second term, by Cauchy-Schwarz inequality, 
\begin{align}
 & \frac{3}{t}\Big|\big(\eb_{k}^{\top}\sum_{s\in\Psi\cap[t]}U_{k,s}\eb_{k}\eb_{k}^{\top}\big)\big(\tilde{\Fb}_{t}(x)-\Fb(x)\big)\Big|\label{eq:deompose1}\\
 & =\frac{3}{t}|\sum_{s\in\Psi\cap[t]}U_{k,s}|\cdot|\tilde{F}_{k,t}(x)-F_{k}(x)|\nonumber 
\end{align}
By the matrix concentration inequality (Lemma~\ref{lem:matrix_hoeffding}), for any $v\in[0,\frac{t}{4K}]$,
with probability at least $1-2\delta/ (Kt^{2})$, 
\begin{align*}
\frac{1}{t}|\sum_{s\in\Psi\cap[t]}U_{k,s}| & =\frac{1}{t}|\sum_{s=1}^{t}U_{k,s}|+\frac{1}{t}|\sum_{s\in[t]\setminus\Psi}U_{k,s}|\\
 & \le\frac{v}{t^{2}}|\sum_{s=1}^{t}\frac{1}{\PP(\tilde{a}_{s}(\tau_{s})=k|\Fcal_{\tau_{s}}^{-})}|+\frac{1}{v}\log\frac{Kt^{2}}{\delta}+\frac{1}{t}|\sum_{s\in[t]\setminus\Psi}U_{k,s}|.
\end{align*}
Because $\PP(\tilde{a}_{s}(\tau_{s})=k|\Fcal_{\tau_{s}}^{-})\ge\frac{1}{2K}$,
setting $v=\sqrt{\frac{t}{2K}\log\frac{Kt^{2}}{\delta}}\le\frac{t}{4K}$
gives 
\begin{align*}
\frac{v}{t^{2}}|\sum_{s=1}^{t}\frac{1}{\PP(\tilde{a}_{s}(\tau_{s})=k|\Fcal_{\tau_{s}}^{-})}|+\frac{1}{v}\log\frac{Kt^{2}}{\delta}\le & \frac{2vK}{t}+\frac{1}{v}\log\frac{Kt^{2}}{\delta}\\
= & 2\sqrt{\frac{2K}{t}\log\frac{Kt^{2}}{\delta}}.
\end{align*}
For the second term, let $\Hcal_{s}^{-}$ denote the sigma algebra
generated by the event in $\Fcal_{\tau_{s}}^{-}$ and the random variables
$a_{1},\ldots,a_{s-1}$ and $\tilde{a}_{1}(\tau_{1}),\ldots,\tilde{a}_{s-1}(\tau_{s-1})$.
For each $s\in[t]$ let 
\[
M_{s}:=\frac{\EE[\sum_{s\in[t]\setminus\Psi}U_{k,s}|\Hcal_{s+1}^{-}]-\EE[\sum_{s\in[t]\setminus\Psi}U_{k,s}|\Hcal_{s}^{-}]}{t},
\]
denotes the martingale differences. Note that $|M_{s}|\le2K$ and $\Mb_{s}=\Ob$
for $s\in\Psi\cap[t]$. 
Thus by the matrix concentration lemma (Lemma~\ref{lem:matrix_hoeffding}), with
probability at least $1-2\delta/(Kt^2)$, for any $v\in[0,\frac{t}{4K}]$,
then, 
\begin{align*}
|\frac{1}{t}\sum_{s\in\Psi^{c}\cap[t]}U_{k,s}| & =|\sum_{s=1}^{t}M_{s}|\\
 & \le v\sum_{s=1}^{t}\EE[M_{s}^{2}|\Hcal_{s}^{-}]|_{2}+\frac{1}{v}\log\frac{Kt^2}{\delta}\\
 & =v|\sum_{s\in[t]\setminus\Psi}\EE[M_{s}^{2}|\Hcal_{s}^{-}]|_{2}+\frac{1}{v}\log\frac{Kt^2}{\delta}.
\end{align*}
Because for $s\in[t]\setminus\Psi$, 
\begin{align*}
\EE[M_{s}^{2}|\Hcal_{s-1}]\le & \frac{1}{t^{2}}\EE[M_{s}^{2}|\Hcal_{s}^{-}]\\
\le & \frac{1}{t^{2}}\EE[\big(\frac{\II(\tilde{a}_{s}=k)}{\PP(\tilde{a}_{s}(\tau_{s})=k|\Fcal_{\tau_{s}}^{-})}|\Hcal_{s}^{-}]\\
= & \frac{1}{t^{2}}\frac{1}{\PP(\tilde{a}_{s}=k|\Fcal_{\tau_{s}}^{-})},
\end{align*}
we obtain $\EE[M_{s}^{2}|\Hcal_{s-1}]\preceq\frac{2K}{t^{2}}.$ Thus,
setting $v=\sqrt{\frac{t}{2K}\log\frac{Kt^{2}}{\delta}}$ gives, 
\begin{align*}
|\frac{1}{t}\sum_{s\in[t]\setminus\Psi}U_{k,s}|\le & \frac{2Kv}{t}+\frac{1}{v}\log\frac{Kt^{2}}{\delta}\\
= & 2\sqrt{\frac{2K}{t}\log\frac{Kt^{2}}{\delta}},
\end{align*}
where the inequality holds because $|[t]\setminus\Psi|\le t$. From
\eqref{eq:deompose1}, 
\[
\frac{1}{t}\Big|\big(\sum_{s\in\Psi\cap[t]}U_{k,s}\eb_{k}\eb_{k}^{\top}\big)\big(\tilde{\Fb}_{t}(x)-\Fb(x)\big)\Big|\le4\sqrt{\frac{2K}{t}\log\frac{Kt^{2}}{\delta}}|\tilde{F}_{k,t}(x)-F_{k}(x)|
\]
Thus, 
\begin{align*}
 & \sup_{x\in\RR}|\widehat{F}_{k,t}(x)-F_{k}(x)|\\
 & \le\frac{3}{t}\sup_{x\in\RR}\Big|\sum_{s=1}^{t}w_{s}\eta_{k,s}(x)\Big|+\frac{3}{t}\sup_{x\in\RR}\Big|\eb_{k}^{\top}\big(\sum_{s\in\Psi\cap[t]}U_{k,s}\eb_{k}\eb_{k}^{\top}\big)\big(\tilde{\Fb}_{t}(x)-\Fb(x)\big)\Big|\\
 & \le9\sqrt{\frac{1}{t}\log\frac{Kt^{3}}{\delta}}+12\sqrt{\frac{2K}{t}\log\frac{Kt^{2}}{\delta}}\sup_{x\in\RR}|\tilde{F}_{k,t}(x)-F_{k}(x)|.
\end{align*}
which completes the proof. 
\end{proof}

\corUniformEstimatorRate*
\label{app:F_conv_proof}
\begin{proof}
Combining the error bound of the exploration-mixed estimator and the
DR estimator, with probability at least $1-9\delta/(t^{2}K)$, for
each $x\in\RR$, 
\begin{align*}
 & \Vert\widehat{F}_{k,t}-F_{k}\Vert_{\infty}\\
 & \le9\sqrt{\frac{1}{t}\log\frac{Kt^{3}}{\delta}}+12\sqrt{\frac{2K}{t}\log\frac{Kt^{2}}{\delta}}\Big(102\sqrt{\frac{3}{t}\log\frac{Kt^{2}}{\delta}}+6\cdot49\sqrt{\frac{1}{|[t]\cap\Ecal|}\log\frac{Kt^{2}}{\delta}}\Big)\\
 & =9\sqrt{\frac{1}{t}\log\frac{Kt^{3}}{\delta}}+3\cdot408\sqrt{\frac{6K}{t^{2}}\log^{2}\frac{Kt^{2}}{\delta}}+3\cdot24\cdot49\sqrt{\frac{2K}{t|[t]\cap\Ecal|}\log^{2}\frac{Kt^{2}}{\delta}}
\end{align*}
Because $|\Ecal\cap[t]|\ge8K\log\frac{Kt^{2}}{\delta}$ for $t\ge T_{e}$
\[
3\cdot24\cdot49\sqrt{\frac{2K}{t|[t]\cap\Ecal|}\log^{2}\frac{Kt^{2}}{\delta}}\le3\cdot12\cdot49\sqrt{\frac{1}{t}\log\frac{Kt^{2}}{\delta}}.
\]
In addition, since $t\ge128\cdot3K\log\frac{Kt^{2}}{\delta}$ we obtain
\[
3\cdot408\sqrt{\frac{6K}{t^{2}}\log^{2}\frac{Kt^{2}}{\delta}}\le153\sqrt{\frac{1}{t}\log\frac{Kt^{2}}{\delta}}.
\]
Thus,
\begin{align*}
\Vert\widehat{F}_{k,t}-F_{k}\Vert_{\infty}\le & 3\sqrt{\frac{1}{t}\log\frac{Kt^{3}}{\delta}}+51\sqrt{\frac{1}{t}\log\frac{Kt^{2}}{\delta}}+12\cdot49\sqrt{\frac{1}{t}\log\frac{Kt^{2}}{\delta}}\\
\le & (9+153+36\cdot49)\sqrt{\frac{1}{t}\log\frac{Kt^{3}}{\delta}},
\end{align*}
which completes the proof.
\end{proof}

\subsection{Sample Complexity Upper Bound}
\label{app:sample_complex_proof}
\thmSampleComplexityDominantArm*

In order to satisfy the exploration condition $|[t]\cap\Ecal|\ge 8K\log\frac{Kt^2}{\delta}$, we need at least
\[
32K(1+\log\frac{16K\sqrt{K}}{\sqrt{\delta}e})
\]
number of samples by Lemma~\ref{lem:logt}.
Without loss of generality, we rearrange the arms in descending order of the dominance score, $D_{1}>\cdots>D_{K}$.
Let $T$ denote the round that the algorithm terminates with $\Acal_T =\{1\}$.
By construction of the algorithm, the arm $a>2$ is eliminated when
\begin{equation}
\widehat{D}_{1,T}>\widehat{D}_{a,T}+\beta_{a,T}+\beta_{1,T}
\label{eq:termination}
\end{equation}
By the convergence of $\widehat{D}_{a,t}$ (Theorem~\ref{thm:D_convergence}), with probability at least $1-10\delta$,
\[
D_a \le \widehat{D}_{a,T} + \beta_{a,T} \le \widehat{D}_{1,T} -\beta_{1,T} \le D_1,
\]
and $a$ is the suboptimal arm.
Note that 
\[
\widehat{D}_{1,T} - \widehat{D}_{a,T} \ge \Delta_a - \beta_{1,T}-\beta_{a,T}.
\]
Thus condition~\eqref{eq:termination} holds when 
\[
\beta_{a,T}+\beta_{1,T} \le  \Delta_a - \beta_{1,T}-\beta_{a,T},
\]
which is equivalent to 
\[
\beta_{a,T}+\beta_{1,T} \le \frac{\Delta_a}{2}.
\]
Because the algorithm selects all arms in the undetermined arms $\Acal_t$, the condition
\[
3\gamma_{T} + \sqrt{\frac{2\log\frac{T^2}{\delta}}{N_{a,T}}} \le \frac{\Delta_a}{4},
\]
implies~\eqref{eq:termination}.
Since $\gamma_{T}\le 1926 \sqrt{\frac{3\log\frac{Kt^2}{\delta}}{2N_{a,T}}}$, 
\[
N_{a,T} \ge \frac{32\cdot (3\cdot 1926 +1)^2}{\Delta_a^2}\log \frac{KT^2}{\delta},
\]
implies~\eqref{eq:termination}.
Summing up over $a>2$ gives
\[
T \ge \sum_{a=2}^{K}\frac{32\cdot (3\cdot 1926 +1)^2}{\Delta_a^2}\log \frac{T^2}{\delta},
\]
which is implied by (using Lemma~\ref{lem:logt}),
\[
T \ge (1+\sum_{a=2}^{K}\frac{4\cdot32\cdot (3\cdot 1926 +1)^2}{\Delta_a^2}) \log \frac{\sqrt{K}\sum_{a=2}^{K}\frac{64\cdot (3\cdot 1926 +1)^2}{\Delta_a^2}}{e\sqrt{\delta}}, 
\]
which completes the proof.

\section{Novel Technical results}

\subsection{Construction of Hard Distributions for the Dominance Score Estimation}
\begin{lemma}
\label{lem:lower_noise} Let $D_{a}$ denote the dominance score for
arm $a\in[K]$. Then there exists a distribution $\PP_{a}$ such that
for any failure probability parameter $\delta\in(0,1/4)$, the maximum
estimation error satisfies: 
\[
|D_{a}-\widehat{D}_{a}|>\epsilon,
\]
with probability at least $\delta$, for any estimator $\widehat{D}_{a,n}$
that uses at most $n\le\frac{\sigma_{a}^{2}}{12\epsilon^{2}}\log\left(\frac{1}{4\delta}\right)$
independent samples $R_{a,1},\ldots R_{a,n}$ sampled from $\mathbb{P}_{a}$.
\end{lemma}

\begin{proof}
\textbf{Step 1. Construction of the parametric noise distribution.}
Given $v\in\{-1,1\}$, we construct the discrete noise distribution
of $\eta_{v}$ such that 
\[
\eta_{v}=\begin{cases}
-\dfrac{\sigma_{a}}{2\left(1+\frac{2\epsilon}{\sigma_{a}}v\right)} & \text{with probability }p_{+}(v):=\frac{1}{2}+\frac{\epsilon}{\sigma_{a}}v,\\[10pt]
\dfrac{\sigma_{a}}{2\left(1-\frac{2\epsilon}{\sigma_{a}}v\right)} & \text{with probability }p_{-}(v):=\frac{1}{2}-\frac{\epsilon}{\sigma_{a}}v.
\end{cases}
\]
Let $\PP_{a,v}$ denote the distribution of reward $R_{a}$ that satisfies,
\[
\sum_{k=1}^{K}F_{k}(R_{a})\II(R_{a}\in\Gcal_{k})=\sum_{k=1}^{K}\int_{\Gcal_{k}}F_{k}(r)d\PP_{a}(r)+\eta_{v}.
\]
Note that $\EE_{v}[\eta_{v}]=0$. Since $\epsilon\in(0,\sigma_{a}/4)$,
we have $|\eta_{v}|\le\sigma$. We set 
\[
\sigma_{a}=\sigma_{a}:=\frac{1}{2}\sum_{k=1}^{K}\int_{\Gcal_{k}}F_{k}(r)d\PP_{a,}(r)=\frac{D_{a}}{2}
\]
 to guarantee 
\[
\sum_{k=1}^{K}\int_{\Gcal_{k}}F_{k}(r)d\PP_{a}(r)+\eta_{v}\ge0.
\]
Furthermore, for all $v\in\{-1,1\}$, the distance between support
points satisfies: 
\begin{equation}
\Big|\frac{\sigma_{a}}{2\left(1-\frac{2\epsilon}{\sigma_{a}}v\right)}+\frac{\sigma_{a}}{2\left(1+\frac{2\epsilon}{\sigma_{a}}v\right)}\Big|>2\epsilon.\label{eq:diff_lower}
\end{equation}

\textbf{Step 2. Reduction to hypothesis testing.} Let $D_{a}:=\sum_{k=1}^{K}\int_{\Gcal_{k}}F_{k}(r)d\PP_{a}(r)$
denote the dominance score over the distribution $\PP_{a,v}$. Let
$\PP_{v}^{(n)}=\prod_{i=1}^{n}\PP_{v}$ denote the joint distribution
of the $n$ observed samples $R_{a,1},\ldots,R_{a,n}$. Consider an
independent fresh sample $f(R_{a,v}):=\sum_{k=1}^{K}F_{k}(R_{a,v})\II(R_{a,v}\in\Gcal_{k})=D_{a}+\eta_{v}$.
For any estimator $\widehat{D}_{a,n}$ that uses $R_{a,1},\ldots,R_{a,n}$,
we bound the probability of small estimation error by: 
\[
\PP_{v}^{(n)}\left(|D_{a}-\widehat{D}_{a,n}|>\epsilon\right)=\PP_{v}^{(n)}\big(|f(R_{a,v})-\eta_{v}-\widehat{D}_{a,n}|>\epsilon\big).
\]
Taking expectation on $\eta_{v}$, 
\[
\PP_{v}^{(n)}\left(|D_{a}-\widehat{D}_{a,n}|>\epsilon\right)=\EE_{v}\big[\PP_{v}^{(n)}\big(|f(R_{a,v})-\widehat{D}_{a,n}-\eta_{v}|>\epsilon|\eta_{v}\big)\big].
\]
Next, we lower-bound the complement event probability: 
\begin{align*}
 & \EE_{v}\big[\PP_{v}^{(n)}\big(|f(R_{a,v})-\widehat{D}_{a,n}-\eta_{v}|>\epsilon|\eta_{v}\big)\big]\\
 & =\PP_{v}^{(n)}\Big(\Big|f(R_{a,v})-\widehat{D}_{a,n}+\dfrac{\sigma_{a}}{2\left(1+\frac{2\epsilon}{\sigma_{a}}v\right)}\Big|>\epsilon\Big)p_{+}(v)\\
 & \quad+\PP_{v}^{(n)}\Big(\Big|f(R_{a,v})-\widehat{D}_{a,n}-\dfrac{\sigma_{a}}{2\left(1-\frac{2\epsilon}{\sigma_{a}}v\right)}\Big|>\epsilon\Big)p_{-}(v)\\
 & \ge\bigg(\PP_{v}^{(n)}\Big(\Big|f(R_{a,v})-\widehat{D}_{a,n}+\dfrac{\sigma_{a}}{2\left(1+\frac{2\epsilon}{\sigma_{a}}v\right)}\Big|>\epsilon\Big)+\PP_{v}^{(n)}\Big(\Big|f(R_{a,v})-\widehat{D}_{a,n}-\dfrac{\sigma_{a}}{2\left(1-\frac{2\epsilon}{\sigma_{a}}v\right)}\Big|>\epsilon\Big)\bigg)\min\left\{ p_{+}(v),p_{-}(v)\right\} .
\end{align*}
Because \eqref{eq:diff_lower} holds, the region 
\[
\RR\subseteq\Big\{ r\in\RR:\Big|r+\dfrac{\sigma_{a}}{2\left(1+\frac{2\epsilon}{\sigma_{a}}v\right)}\Big|>\epsilon\Big\}\cup\Big\{ r\in\RR:\Big|r-\dfrac{\sigma_{a}}{2\left(1-\frac{2\epsilon}{\sigma_{a}}v\right)}\Big|>\epsilon\Big\},
\]
and it follows that
\begin{equation}
\EE_{v}\big[\PP_{v}^{(n)}\big(|f(R_{a,v})-\widehat{D}_{a,n}-\eta_{v}|>\epsilon|\eta_{v}\big)\big]\ge\min\left\{ p_{+}(v),p_{-}(v)\right\} .\label{eq:P_v_lower}
\end{equation}
Define the optimal decision function $\widehat{f}(r_{a},z)$ for $z\in\{-1,1\}$
and $r_{a}\in\RR$ as: 
\[
\hat{D}(r_{a},z):=\begin{cases}
f(r_{a})+\frac{\sigma_{a}}{2(1+2v\epsilon/\sigma_{a})} & \text{if }p_{+}(z)>p_{-}(z),\\
f(r_{a})-\frac{\sigma_{a}}{2(1-2v\epsilon/\sigma_{a})} & \text{otherwise}.
\end{cases}
\]
Substituting the worst-case minimizer into \eqref{eq:P_v_lower} yields:
\[
\EE_{v}\big[\PP_{v}^{(n)}\big(|f(R_{a,v})-\hat{D}(R_{a,v},v)-\eta_{v}|>\epsilon|\eta_{v}\big)\big]=\min\left\{ p_{+}(v),p_{-}(v)\right\} .
\]
Thus, for any estimator $\widehat{D}_{a,n}$,
\[
\PP_{v}^{(n)}\left(|D_{a}-\widehat{D}_{a,n}|\le\epsilon\right)\le\PP_{v}^{(n)}\left(|D_{a}-\widehat{f}(R_{a,v},v)|\le\epsilon\right).
\]
For any decision rule $\widehat{v}\in\{-1,1\}$ based on $R_{a,1},\ldots,R_{a,n}$,
the value $\widehat{f}(R_{a,v},\widehat{v})$ is the estimator for
$D_{a}$. It follows that
\[
\PP_{v}^{(n)}\left(|D_{a}-\widehat{D}_{a,n}|\le\epsilon\right)\le\PP_{v}^{(n)}\left(\big\{|D_{a}-\widehat{f}(R_{a,v},\widehat{v})|\le\epsilon\big\}\cap\big\{\widehat{v}=v\big\}\right)\le\PP_{v}^{(n)}(\widehat{v}=v).
\]
Taking the complement, we lower-bound the estimation error by the
hypothesis testing error: 
\[
\PP_{v}^{(n)}\left(|D_{a}-\widehat{D}_{a,n}|>\epsilon\right)\ge\PP_{v}^{(n)}(\widehat{v}\neq v).
\]
Taking the supremum over all target hypercube configurations $v\in\{-1,1\}$:
\begin{align*}
\sup_{v\in\{-1,1\}}\PP_{v}^{(n)}\left(|D_{a}-\widehat{D}_{a,n}|>\epsilon\right) & \ge\sup_{v\in\{-1,1\}}\PP_{v}^{(n)}(\widehat{v}\neq v)\\
 & \ge\frac{1}{2}\big(\PP_{1}^{(n)}(\widehat{v}=-1)+\PP_{-1}^{(n)}(\widehat{v}=1)\big)\\
 & =\frac{1}{2}\big(1-\PP_{1}^{(n)}(\widehat{v}=1)+\PP_{-1}^{(n)}(\widehat{v}=1)\big)\\
 & \ge\frac{1-TV(\PP_{1}^{(n)},\PP_{-1}^{(n)})}{2}.
\end{align*}
Applying the Bretagnolle--Huber inequality (Lemma~\ref{lem:bretagnolle_huber}):
\[
\sup_{v\in\{-1,1\}}\PP_{v}^{(n)}\left(|D_{a}-\widehat{D}_{a,n}|>\epsilon\right)\ge\frac{1}{4}\exp\left(-n\text{KL}\left(\PP_{1}^{(n)},\PP_{-1}^{(n)}\right)\right).
\]
Using the explicit KL-divergence formula for binary distributions:
\begin{align*}
\text{KL}\left(\PP_{1}^{(n)},\PP_{-1}^{(n)}\right) & =\left(\frac{1}{2}+\frac{\epsilon}{\sigma_{a}}\right)\log\left(\frac{\frac{1}{2}+\frac{\epsilon}{\sigma_{a}}}{\frac{1}{2}-\frac{\epsilon}{\sigma_{a}}}\right)+\left(\frac{1}{2}-\frac{\epsilon}{\sigma_{a}}\right)\log\left(\frac{\frac{1}{2}-\frac{\epsilon}{\sigma_{a}}}{\frac{1}{2}+\frac{\epsilon}{\sigma_{a}}}\right)\\
 & =\frac{2\epsilon}{\sigma_{a}}\log\left(\frac{1+\frac{2\epsilon}{\sigma_{a}}}{1-\frac{2\epsilon}{\sigma_{a}}}\right)\le\frac{2\epsilon}{\sigma_{a}}\left(\frac{2\epsilon}{\sigma_{a}}+\frac{\frac{2\epsilon}{\sigma_{a}}}{1-\frac{2\epsilon}{\sigma_{a}}}\right)\le\frac{12\epsilon^{2}}{\sigma_{a}^{2}},
\end{align*}
where the inequalities hold since $\epsilon\in(0,\sigma_{a}/4)$.
Therefore: 
\[
\sup_{v\in\{-1,1\}}\PP_{v}^{(n)}\left(|D_{a}-\widehat{D}_{a,n}|>\epsilon\right)\ge\frac{1}{4}\exp\left(-\frac{12n\epsilon^{2}}{\sigma_{a}^{2}}\right).
\]

\textbf{Step 4. Sample size requirement.} Setting $n\le\frac{\sigma_{a}^{2}}{12\epsilon^{2}}\log\left(\frac{1}{4\delta}\right)$
ensures that the failure probability is lower-bounded by $\delta$:
\[
\sup_{v\in\{-1,1\}}\PP_{v}^{(n)}\left(|D_{a}-\widehat{D}_{a,n}|>\epsilon\right)\ge\delta,
\]
which is implied by $n\le\frac{\sigma_{a}^{2}}{12\epsilon^{2}}\log\frac{1}{4\delta}.
$ 
\end{proof}

\subsection{Importance Weighting Bound}
 \begin{lemma}[Importance Weighting Bound]
\label{lem:importance_weight}
Given $\delta\in(0,1)$, for any $t$ that satisfies $t\ge8K\log\frac{Kt^2}{\delta}$, suppose that $\{\Xb_{k,s}\}_{k\in[K],\,s\in[t]}$ are non-negative deterministic matrices satisfying 
$\bigl\| \sum_{k=1}^{K} \Xb_{k,s} \bigr\|_2 \le b$ for every $s\in[t]$. 
Then, with probability at least $1 - \frac{4\delta}{t^2}$, the following bound holds:
\[
\left\| \sum_{s \in [t]} w_s \Xb_{a_s, s} - \sum_{s \in [t] \cap \Psi} \sum_{k=1}^{K} \Xb_{k,s} \right\|_2 
\le b |[t] \setminus \Psi| + 4b \sqrt{2K t \log \left( \frac{K t^2}{\delta} \right)}.
\]
\end{lemma}

\begin{proof}
By definition of $w_{s}$, 
\[
\sum_{s\in[t]}w_{s}\Xb_{a_{s},s}=\sum_{s\in[t]\cap\Psi}\frac{1}{1/2}\Xb_{a_{s},s}+\sum_{s\in[t]\cap\Psi^{c}}\Xb_{a_{s},s}.
\]
On the rounds $\Psi^{c}$, because $\Xb_{a_{s},s}$ is nonnegative
definite, 
\[
\big\Vert\sum_{s\in[t]\cap\Psi^{c}}\Xb_{a_{s},s}\big\Vert_{2}\le|[t]\cap\Psi^{c}|b.
\]
On the rounds $s\in\Psi$, and $n\in[\rho_{s}]$, let $\tilde{\Fcal}_{n,s}$
denote the sigma algebra generated by the random variables $\tilde{a}_{s}(1),\ldots,\tilde{a}_{s}(n)$.
For the stopping time $\tau_{s}:=\inf\{n\in[\rho_{s}]:\tilde{a}_{s}(n)=a_{s}\}$,
we define the sigma algebra $\Fcal_{\tau_{s}}^{-}:=\sigma\{A\cap\{\tau_{s}=n\}:A\in\tilde{\Fcal}_{n-1,s},\forall n\ge1\}$.
Then it follows that $\PP(\tilde{a}_{s}(\tau_{s})=a_{s}|\Fcal_{\tau_{s}}^{-})=1/2$.
Thus, 
\begin{align*}
\sum_{s\in[t]\cap\Psi}\frac{1}{1/2}\Xb_{a_{s},s}= & \sum_{s\in[t]\cap\Psi}\frac{1}{\PP(\tilde{a}_{s}(\tau_{s})=a_{s}|\Fcal_{\tau_{s}}^{-})}\Xb_{a_{s},s}\\
= & \sum_{s\in[t]\cap\Psi}\sum_{k=1}^{K}\frac{\II(a_{s}=k)}{\PP(\tilde{a}_{s}(\tau_{s})=k|\Fcal_{\tau_{s}}^{-})}\Xb_{k,s}\\
= & \sum_{s\in[t]\cap\Psi}\sum_{k=1}^{K}\frac{\II(\tilde{a}_{s}(\tau_{s})=k)}{\PP(\tilde{a}_{s}(\tau_{s})=k|\Fcal_{\tau_{s}}^{-})}\Xb_{k,s}.
\end{align*}
Decomposing into two parts,
\[
\sum_{s\in[t]\cap\Psi}\frac{1}{1/2}\Xb_{a_{s},s}=\sum_{s\in[t]}\sum_{k=1}^{K}\frac{\II(\tilde{a}_{s}(\tau_{s})=k)}{\PP(\tilde{a}_{s}(\tau_{s})=k|\Fcal_{\tau_{s}}^{-})}\Xb_{k,s}-\sum_{s\in[t]\setminus\Psi}\sum_{k=1}^{K}\frac{\II(\tilde{a}_{s}(\tau_{s})=k)}{\PP(\tilde{a}_{s}(\tau_{s})=k|\Fcal_{\tau_{s}}^{-})}\Xb_{k,s}
\]
Define $\Db_{s}:=\sum_{k=1}^{K}(\frac{\II(\tilde{a}_{s}(\tau_{s})=k)}{\PP(\tilde{a}_{s}(\tau_{s})=k|\Fcal_{\tau_{s}}^{-})}-1)\Xb_{k,s}$.
It follows that 
\[
\Big\Vert\sum_{s\in[t]}w_{s}\Xb_{a_{s},s}-\sum_{s\in[t]\cap\Psi}\sum_{k=1}^{K}\Xb_{k,s}\Big\Vert_{2}\le\Big\Vert\sum_{s\in[t]}\Db_{s}\Big\Vert_{2}+|[t]\cap\Psi^{c}|b+\Big\Vert\sum_{s\in[t]\cap\Psi^{c}}\Db_{s}\Big\Vert_{2}.
\]
Let $\Hcal_{s}^{-}$ denote the sigma algebra generated by the event
in $\Fcal_{\tau_{s}}^{-}$ and the random variables $a_{1},\ldots,a_{s-1}$
and $\tilde{a}_{1}(\tau_{1}),\ldots,\tilde{a}_{s-1}(\tau_{s-1})$.
It is easy to show that $\|\Db_{s}\|_{2}\le2Kb$. For the first term,
by the matrix concentration inequality (Lemma~\ref{lem:matrix_hoeffding}), for $v\in(0,\frac{1}{4bK}],$
with probability at least $1-2\delta/t^{2}$, 
\[
\Big\Vert\sum_{s\in[t]}\Db_{s}\Big\Vert_{2}\le v\Big\Vert\sum_{s\in[t]}\EE[\Db_{s}^{2}|\Hcal_{s}^{-}]\Big\Vert_{2}+\frac{1}{v}\log\frac{Kt^{2}}{\delta}.
\]
Because the second moment is greater than the variance, 
\begin{align*}
\EE[\Db_{s}^{2}|\Hcal_{s}^{-}]= & \VV\Big[\sum_{k=1}^{K}\frac{\II(\tilde{a}_{s}(\tau_{s})=k)}{\PP(\tilde{a}_{s}(\tau_{s})=k|\Fcal_{\tau_{s}}^{-})}\Xb_{k,s}\Big|\Hcal_{s}^{-}\Big]\\
\preceq & \EE\Big[\Big(\sum_{k=1}^{K}\frac{\II(\tilde{a}_{s}(\tau_{s})=k)}{\PP(\tilde{a}_{s}(\tau_{s})=k|\Fcal_{\tau_{s}}^{-})}\Xb_{k,s}\Big)^{2}\Big|\Hcal_{s}^{-}\Big].
\end{align*}
Because $\II(\tilde{a}_{s}(\tau_s)=k)\II(\tilde{a}_{s}(\tau_s)=k^{\prime})
=\II(\tilde{a}_{s}(\tau_s)=k)\II(k=k^{\prime})$
we obtain, 
\[
\EE[\Db_{s}^{2}|\Hcal_{s}^{-}]\preceq\EE\Big[\sum_{k=1}^{K}\frac{\II(\tilde{a}_{s}(\tau_{s})=k)}{\PP(\tilde{a}_{s}(\tau_{s})=k|\Fcal_{\tau_{s}}^{-})^{2}}\Xb_{k,s}^{2}\Big|\Hcal_{s}^{-}\Big]=\sum_{k=1}^{K}\frac{1}{\PP(\tilde{a}_{s}(\tau_{s})=k|\Fcal_{\tau_{s}}^{-})}\Xb_{k,s}^{2}.
\]
Because $\Db_{s}^{2}$ is nonnegative definite, 
\[
\Big\Vert\sum_{s\in[t]}\EE[\Db_{s}^{2}|\Hcal_{s}^{-}]\Big\Vert_{2}\le\Big\Vert\sum_{s\in[t]}\sum_{k=1}^{K}\frac{\Xb_{k,s}^{2}}{\PP(\tilde{a}_{s}(\tau_{s})=k|\Fcal_{\tau_{s}}^{-})}\Big\Vert_{2}\le2Ktb^{2}.
\]
Note that $t\ge T_{e}$ implies $t\ge8K\log\frac{K(t+1)^{2}}{\delta}$.
Setting $v=\frac{\sqrt{\log\frac{Kt^{2}}{\delta}}}{b\sqrt{2Kt}}\le\frac{1}{4bK}$
gives,
\[
\Big\Vert\sum_{s\in[t]}\Db_{s}\Big\Vert_{2}\le2b\sqrt{2Kt\log\frac{Kt^{2}}{\delta}}.
\]
For the other term, for each $u\in[t]$, define
\[
\Mb_{u}:=\EE\Big[\sum_{s\in[t]\setminus\Psi}\Db_{s}\Big|\Hcal_{u}^{-}\Big]-\EE\Big[\sum_{s\in[t]\setminus\Psi}\Db_{s}\Big|\Hcal_{u-1}^{-}\Big]
\]
denote the martingale difference. Note that $\Mb_{u}=O$ if $u\in[t]\cap\Psi$
and $\Mb_{u}=\Db_{u}$ if $u\in[t]\setminus\Psi$. Thus, for any
$v\in(0,\frac{1}{4Kb}]$, with probability at least $1-2\delta/t^{2}$,
\begin{align*}
\Big\Vert\sum_{s\in[t]\setminus\Psi}\Db_{s}\Big\Vert_{2}= & \Big\Vert\sum_{u\in[t]}\Mb_{u}\Big\Vert_{2}\\
\le & v\Big\Vert\sum_{u\in[t]}\EE[\Mb_{u}^{2}|\Hcal_{u-1}^{-}]\Big\Vert_{2}+\frac{1}{v}\log\frac{Kt^{2}}{\delta}\\
= & v\Big\Vert\sum_{s\in[t]\setminus\Psi}\EE[\Mb_{s}^{2}|\Hcal_{s-1}^{-}]\Big\Vert_{2}+\frac{1}{v}\log\frac{Kt^{2}}{\delta}\\
\le & 2vKb^{2}t+\frac{1}{v}\log\frac{Kt^{2}}{\delta},
\end{align*}
where the last inequality uses the fact that $\Vert\EE[\Mb_{u}^{2}|\Hcal_{u-1}^{-}]\Vert_{2}=\Vert\EE[\Db_{u}^{2}|\Hcal_{u-1}^{-}]\Vert_{2}\le2Kb^{2}$.
Setting $v=\frac{\sqrt{\log\frac{Kt^{2}}{\delta}}}{b\sqrt{2Kt}}\le\frac{1}{4bK}$
gives,
\[
\Big\Vert\sum_{s\in[t]\setminus\Psi}\Db_{s}\Big\Vert_{2}\le2b\sqrt{2Kt\log\frac{Kt^{2}}{\delta}}
\]
Thus, with probability at least $1-4\delta/t^{2}$, 
\[
\Big\Vert\sum_{s\in[t]}w_{s}\Xb_{a_{s},s}-\sum_{s\in[t]\cap\Psi}\sum_{k=1}^{K}\Xb_{k,s}\Big\Vert_{2}\le b|[t]\setminus\Psi|+4b\sqrt{2Kt\log\frac{Kt^{2}}{\delta}},
\]
which completes the proof. 
\end{proof}

\subsection{Variance-aware Matrix Concentration Inequality}

\global\long\def\CE#1#2{\mathbb{E}\left[\left.#1\right\vert #2\right]}%
\global\long\def\CP#1#2{\mathbb{P}\left(\left.#1\right\vert #2\right)}%

\begin{lemma}[Matrix Hoeffding Inequality]
\label{lem:matrix_hoeffding} Let $\{M_{\tau}\}_{\tau=1}^{t}$ be
a matrix-valued stochastic process adapted to a filtration $\{\mathcal{F}_{\tau}\}$,
and let $\Delta_{\tau}:=M_{\tau}-\mathbb{E}[M_{\tau}\mid\mathcal{F}_{\tau-1}]$
be the associated martingale difference matrices. Suppose $\|\Delta_{\tau}\|_{2}\le b$
almost surely for all $\tau$. Then, for any $u>0$ and $v\in\left(0,\tfrac{1}{2b}\right]$,
with probability at least $1-\delta$, 
\[
\left\Vert \sum_{\tau=1}^{t}\Delta_{\tau}\right\Vert _{2}\le v\left\Vert \sum_{\tau=1}^{t}\mathbb{E}[\Delta_{\tau}^{2}\mid\mathcal{F}_{\tau-1}]\right\Vert _{2}+\frac{1}{v}\log\frac{2d}{\delta}.
\]
\end{lemma}

\begin{proof}
The proof adapts Bernstein's inequality to a matrix stochastic process,
adopting the techniques in~\citet{tropp2012user}. 
Write $\mathbb{E}_{\tau}[\cdot]:=\mathbb{E}[\cdot\mid\mathcal{F}_{\tau-1}]$.
We aim to bound, for $u>0$ and $v\in\left(0,\tfrac{1}{2b}\right]$,
\[
\mathbb{P}\!\left(\left\Vert v\sum_{\tau=1}^{t}\Delta_{\tau}-v^{2}\sum_{\tau=1}^{t}\mathbb{E}_{\tau}[\Delta_{\tau}^{2}]\right\Vert _{2}\ge u\right).
\]

\paragraph{Step 1 (Union bound).}

Using $\|A\|_{2}=\max\{\lambda_{\max}(A),\lambda_{\max}(-A)\}$ and
the union bound, 
\begin{align*}
 & \mathbb{P}\!\left(\left\Vert v\sum_{\tau=1}^{t}\Delta_{\tau}-v^{2}\sum_{\tau=1}^{t}\mathbb{E}_{\tau}[\Delta_{\tau}^{2}]\right\Vert _{2}\ge u\right)\\
 & \le\mathbb{P}\!\left(\lambda_{\max}\!\left(v\sum_{\tau=1}^{t}\Delta_{\tau}-v^{2}\sum_{\tau=1}^{t}\mathbb{E}_{\tau}[\Delta_{\tau}^{2}]\right)\ge u\right)+\mathbb{P}\!\left(\lambda_{\max}\!\left(-v\sum_{\tau=1}^{t}\Delta_{\tau}+v^{2}\sum_{\tau=1}^{t}\mathbb{E}_{\tau}[\Delta_{\tau}^{2}]\right)\ge u\right).
\end{align*}
By symmetry it suffices to bound the first term; the second is treated
identically.

\paragraph{Step 2 (Markov's inequality).}

Since $\exp$ is monotone and by Markov's inequality, 
\[
\mathbb{P}\!\left(\lambda_{\max}\!\left(v\sum_{\tau=1}^{t}\Delta_{\tau}-v^{2}\sum_{\tau=1}^{t}\mathbb{E}_{\tau}[\Delta_{\tau}^{2}]\right)\ge u\right)\le e^{-u}\,\mathbb{E}\!\left[\exp\!\left(\lambda_{\max}\!\left(v\sum_{\tau=1}^{t}\Delta_{\tau}-v^{2}\sum_{\tau=1}^{t}\mathbb{E}_{\tau}[\Delta_{\tau}^{2}]\right)\right)\right].
\]

\paragraph{Step 3 (Bounding the moment generating function).}

For any symmetric matrix $A$, $\exp(\lambda_{\max}(A))=\lambda_{\max}(\exp(A))\le\operatorname{tr}(\exp(A))$,
so 
\[
\mathbb{E}\!\left[\exp\!\left(\lambda_{\max}\!\left(v\sum_{\tau=1}^{t}\Delta_{\tau}-v^{2}\sum_{\tau=1}^{t}\mathbb{E}_{\tau}[\Delta_{\tau}^{2}]\right)\right)\right]=\mathbb{E}\!\left[\operatorname{tr}\!\left(\exp\!\left(v\sum_{\tau=1}^{t}\Delta_{\tau}-v^{2}\sum_{\tau=1}^{t}\mathbb{E}_{\tau}[\Delta_{\tau}^{2}]\right)\right)\right].
\]

Writing the exponent as $v\sum_{\tau=1}^{t-1}\Delta_{\tau}-v^{2}\sum_{\tau=1}^{t}\mathbb{E}_{\tau}[\Delta_{\tau}^{2}]+\log\exp(v\Delta_{t})$
and applying Jensen's inequality (the map $H\mapsto\operatorname{tr}(\exp(A+H))$
is convex) together with the tower property gives 
\begin{align*}
 & \mathbb{E}\!\left[\operatorname{tr}\!\left(\exp\!\left(v\sum_{\tau=1}^{t}\Delta_{\tau}-v^{2}\sum_{\tau=1}^{t}\mathbb{E}_{\tau}[\Delta_{\tau}^{2}]\right)\right)\right]\\
 & \le\mathbb{E}\!\left[\operatorname{tr}\!\left(\exp\!\left(v\sum_{\tau=1}^{t-1}\Delta_{\tau}-v^{2}\sum_{\tau=1}^{t}\mathbb{E}_{\tau}[\Delta_{\tau}^{2}]+\log\mathbb{E}_{t-1}[\exp(v\Delta_{t})]\right)\right)\right].
\end{align*}

We now establish the operator inequality 
\begin{equation}
\mathbb{E}_{t-1}[\exp(v\Delta_{t})]\preceq\exp\!\left(v^{2}\,\mathbb{E}_{t-1}[\Delta_{t}^{2}]\right).\label{eq:subgaussian_op}
\end{equation}
Since $\|\Delta_{t}\|_{2}\le b$ and $v\le\tfrac{1}{2b}$, we strictly
have $\|v\Delta_{t}\|_{2}\le\tfrac{1}{2}$. For any real number $s$
such that $|s|\le\tfrac{1}{2}$, the scalar inequality $e^{s}\le1+s+s^{2}$
mathematically holds. This directly implies the operator bound $\exp(v\Delta_{t})\preceq I+v\Delta_{t}+v^{2}\Delta_{t}^{2}$.
Taking the conditional expectation and using $\mathbb{E}_{t-1}[\Delta_{t}]=O$
yields $\mathbb{E}_{t-1}[\exp(v\Delta_{t})]\preceq I+v^{2}\mathbb{E}_{t-1}[\Delta_{t}^{2}]$.
Finally, the standard inequality $I+X\preceq\exp(X)$ for any Hermitian
$X$ gives $I+v^{2}\mathbb{E}_{t-1}[\Delta_{t}^{2}]\preceq\exp(v^{2}\mathbb{E}_{t-1}[\Delta_{t}^{2}])$,
which precisely establishes \eqref{eq:subgaussian_op}.

Since $A\preceq B$ implies $\log A\preceq\log B$ (operator monotonicity
of $\log$) and the map $H\mapsto\operatorname{tr}(\exp(C+H))$ is
monotone increasing under the Loewner order, substituting~\eqref{eq:subgaussian_op}
gives 
\begin{align*}
 & \mathbb{E}\!\left[\operatorname{tr}\!\left(\exp\!\left(v\sum_{\tau=1}^{t-1}\Delta_{\tau}-v^{2}\sum_{\tau=1}^{t}\mathbb{E}_{\tau}[\Delta_{\tau}^{2}]+\log\mathbb{E}_{t-1}[\exp(v\Delta_{t})]\right)\right)\right]\\
 & \le\mathbb{E}\!\left[\operatorname{tr}\!\left(\exp\!\left(v\sum_{\tau=1}^{t-1}\Delta_{\tau}-v^{2}\sum_{\tau=1}^{t}\mathbb{E}_{\tau}[\Delta_{\tau}^{2}]+v^{2}\mathbb{E}_{t-1}[\Delta_{t}^{2}]\right)\right)\right]\\
 & =\mathbb{E}\!\left[\operatorname{tr}\!\left(\exp\!\left(v\sum_{\tau=1}^{t-1}\Delta_{\tau}-v^{2}\sum_{\tau=1}^{t-1}\mathbb{E}_{\tau}[\Delta_{\tau}^{2}]\right)\right)\right],
\end{align*}
where the equality uses $-v^{2}\sum_{\tau=1}^{t}\mathbb{E}_{\tau}[\Delta_{\tau}^{2}]+v^{2}\mathbb{E}_{t-1}[\Delta_{t}^{2}]=-v^{2}\sum_{\tau=1}^{t-1}\mathbb{E}_{\tau}[\Delta_{\tau}^{2}]$.
Repeating this peeling argument for $\tau=t-1,t-2,\ldots,1$ yields
\[
\mathbb{E}\!\left[\operatorname{tr}\!\left(\exp\!\left(v\sum_{\tau=1}^{t}\Delta_{\tau}-v^{2}\sum_{\tau=1}^{t}\mathbb{E}_{\tau}[\Delta_{\tau}^{2}]\right)\right)\right]\le\mathbb{E}\!\left[\operatorname{tr}(\exp(0))\right]=\operatorname{tr}(I_{d})=d.
\]

\paragraph{Step 4 (Conclusion).}

Combining Steps 1--3 and applying the same argument to the second
term gives 
\[
\mathbb{P}\!\left(\left\Vert v\sum_{\tau=1}^{t}\Delta_{\tau}-v^{2}\sum_{\tau=1}^{t}\mathbb{E}[\Delta_{\tau}^{2}\mid\mathcal{F}_{\tau-1}]\right\Vert _{2}\ge u\right)\le2d\,e^{-u}.
\]
Setting $2d\,e^{-u}=\delta$ gives $u=\log(2d/\delta)$, so with probability
at least $1-\delta$, 
\[
\left\Vert v\sum_{\tau=1}^{t}\Delta_{\tau}-v^{2}\sum_{\tau=1}^{t}\mathbb{E}[\Delta_{\tau}^{2}\mid\mathcal{F}_{\tau-1}]\right\Vert _{2}\le\log\frac{2d}{\delta}.
\]
Rearranging yields 
\[
\left\Vert \sum_{\tau=1}^{t}\Delta_{\tau}\right\Vert _{2}\le v\left\Vert \sum_{\tau=1}^{t}\mathbb{E}[\Delta_{\tau}^{2}\mid\mathcal{F}_{\tau-1}]\right\Vert _{2}+\frac{1}{v}\log\frac{2d}{\delta}.
\]
\end{proof}

\subsection{Uniform bound for the vector-weighted distribution function residual}

\begin{lemma}[From Point-wise to Uniform Error Bound for the vector-weighted distribution function residual]
\label{lem:vector_pointwise_to_uniform_vcdf}Let $F:\mathbb{R}\to[0,1]$
be a cumulative distribution function, and let $R_{1},\ldots,R_{n}$
denote the random variable from the distribution $F$. Let $v_{1},\ldots,v_{K}\in\RR^{K}$
denote the deterministic vectors such that $v_{i}=(v_{i1},\ldots,v_{iK})^{\top}$.
Define $\eta_{i}(x):=\II(R_{i}\le x)-F(x)$. Suppose that for any
fixed $x_{1},\ldots,x_{n}\in\mathbb{R}$ and any failure probability
$\delta'\in(0,1)$, $\widehat{F}$ satisfies the pointwise bound 
\[
\mathbb{P}\left(\Big\Vert\Big(\sum_{i=1}^{n}\eta_{i}(x_{i1})v_{i1},\ldots\sum_{i=1}^{n}\eta_{i}(x_{iK})v_{iK}\Big)\Big\Vert_{2}\le\epsilon(\delta')\right)\ge1-\delta',
\]
where $\epsilon(\cdot)$ is a non-increasing function of the confidence
level. 
Then, for any choice of $M\in\mathbb{N}$ and any target
failure probability $\delta\in(0,1)$, setting $\delta'=\frac{\delta}{dM}$,
the uniform bound holds:
\[
\mathbb{P}\left(\sup_{x\in\mathbb{R}}\Big\Vert\sum_{i=1}^{n}\eta_{i}(x)v_{i}\Big\Vert_{2}\le\epsilon\left(\frac{\delta}{M}\right)+\frac{\sqrt{\sum_{i=1}^{n}\|v_{i}\|_{2}^{2}}}{M}\right)\ge1-\delta.
\]
\end{lemma}

\begin{proof}
Set $\epsilon_{0}=\frac{1}{M}$. Construct a finite grid of $M+1$
points $-\infty=x_{0}<x_{1}<x_{2}<\dots<x_{M}=\infty$ such that
\[
F(x_{k})-F(x_{k-1})\le\epsilon_{0}=\frac{1}{M}\quad\text{for all }k\in\{1,2,\dots,M\}.
\]
Note that such a partition can always be constructed by choosing $x_{k}=\inf\{x\in\mathbb{R}:F(x)\ge k/M\}$.
Consider the finite set of interior grid points $\{x_{1},x_{2},\dots,x_{M-1}\}$.
For a target failure probability $\delta\in(0,1)$, set $\delta'=\frac{\delta}{M-1}$. Define
the bad event at a point $x_{k}$ as 
\[
A_{k}:=\left\{ \Big\Vert\sum_{i=1}^{n}\eta_{i}(x_{k})v_{i}\Big\Vert_{2}>\epsilon(\delta')\right\} .
\]
By the given pointwise guarantee, $\mathbb{P}(A_{k})\le\delta'$.
Applying the Union Bound over all $M-1$ interior points yields
\[
\mathbb{P}\left(\bigcup_{k=1}^{M-1}A_{k}\right)\le\sum_{k=1}^{M-1}\mathbb{P}(A_{k})\le(M-1)\cdot\delta'=\delta.
\]
Thus, with probability at least $1-\delta$, the good event $\bigcap_{k=1}^{M-1}A_{k}^{c}$
holds, where
\[
\Big\Vert\sum_{i=1}^{n}\eta_{i}(x_{k})v_{i}\Big\Vert_{2}\le\epsilon(\delta')\quad\text{for all }k\in\{1,2,\dots,M-1\}.
\]
At the boundary points $x_{0}=-\infty$ and $x_{M}=\infty$, $\eta_{i}(x_{0})=0$
and $\eta_{i}(x_{M})=1$ for all $i=1,\ldots,N$. Thus, the inequality
holds trivially for $k=0$ and $k=M$. Suppose the event $E$ holds.
Take any arbitrary point $x\in\mathbb{R}$. There exists an index
$k\in\{1,2,\dots,M\}$ such that $x\in[x_{k-1},x_{k}]$. Let $S_{j}:=\sum_{i=1}^{n}v_{ij}$.
Writing $v_{i}:=(v_{i1},\ldots,v_{iK})^{\top}$,
\[
\Big\Vert\sum_{i=1}^{n}\eta_{i}(x)v_{i}\Big\Vert_{2}=\Big\Vert\Big(\sum_{i=1}^{n}\II(R_{i}\le x)v_{i1}-S_{1}F(x),\ldots,\sum_{i=1}^{n}\II(R_{i}\le x)v_{iK}-S_{K}F(x)\Big)^{\top}\Big\Vert_{2}
\]
We decompose the error based on the sign of the vector entries. For
$j=1,\ldots,K$, define $\Ical_{j}^{+}:=\{i\in[n]:v_{ij}\ge0\}$ and
$\Ical_{j}^{-}:=\{i\in[n]:v_{ij}<0\}$. For $j\in\Scal^{+}:=\{j=1,\ldots,K:\sum_{i=1}^{n}\II(R_{i}\le x)v_{ij}\ge S_{j}F(x)\}$
\begin{align*}
0\le & \sum_{i=1}^{n}\II(R_{i}\le x)v_{ij}-S_{j}F(x)\\
= & \sum_{i\in\Ical_{j}^{-}}\II(R_{i}\le x)v_{ij}+\sum_{i\in\Ical_{j}^{+}}\II(R_{i}\le x)v_{ij}-S_{j}F(x)\\
\le & \sum_{i\in\Ical_{j}^{-}}\II(R_{i}\le x_{k-1})v_{ij}+\sum_{i\in\Ical_{j}^{+}}\II(R_{i}\le x_{k})v_{ij}-S_{j}F(x)\\
= & \sum_{i\in\Ical_{j}^{-}}\big(\eta_{i}(x_{k-1})+F(x_{k-1})-F(x)\big)v_{ij}+\sum_{i\in\Ical_{j}^{+}}\big(\eta_{i}(x_{k})+F(x_{k})-F(x)\big)v_{ij}\\
\le & \sum_{i\in\Ical_{j}^{-}}\eta_{i}(x_{k-1})v_{ij}+\sum_{i\in\Ical_{j}^{+}}\eta_{i}(x_{k})v_{ij}+\sum_{i=1}^{n}|v_{ij}|\epsilon_{0}.
\end{align*}
On the other hand, for $j\in\Scal^{-}:=\{j=1,\ldots,K:\sum_{i=1}^{n}\II(R_{i}\le x)v_{ij}<S_{j}F(x)\}$,
it is easy to show that
\begin{align*}
0\le & S_{j}F(x)-\sum_{i=1}^{n}\II(R_{i}\le x)v_{ij}\\
\le & -\sum_{i\in\Ical_{j}^{-}}\eta_{i}(x_{k})v_{ij}-\sum_{i\in\Ical_{j}^{+}}\eta_{i}(x_{k-1})v_{ij}+\sum_{i=1}^{n}|v_{ij}|\epsilon_{0}.
\end{align*}
Thus,
\begin{align*}
\Big\Vert\sum_{i=1}^{n}\eta_{i}(x)v_{i}\Big\Vert_{2}\le\Big( & \sum_{j\in\Scal^{-}}\Big(-\sum_{i\in\Ical_{j}^{-}}\eta_{i}(x_{k})v_{ij}-\sum_{i\in\Ical_{j}^{+}}\eta_{i}(x_{k-1})v_{ij}+\sum_{i=1}^{n}|v_{ij}|\epsilon_{0}\Big)^{2}\\
 & +\sum_{j\in\Scal^{+}}\Big(\sum_{i\in\Ical_{j}^{-}}\eta_{i}(x_{k-1})v_{ij}+\sum_{i\in\Ical_{j}^{+}}\eta_{i}(x_{k})v_{ij}+\sum_{i=1}^{n}|v_{ij}|\epsilon_{0}\Big)^{2}\Big)^{1/2}\\
\le & \sqrt{\sum_{j\in\Scal^{-}}\Big(-\sum_{i\in\Ical_{j}^{-}}\eta_{i}(x_{k})v_{ij}-\sum_{i\in\Ical_{j}^{+}}\eta_{i}(x_{k-1})v_{ij}\Big)^{2}+\sum_{j\in\Scal^{+}}\Big(\sum_{i\in\Ical_{j}^{-}}\eta_{i}(x_{k-1})v_{ij}+\sum_{i\in\Ical_{j}^{+}}\eta_{i}(x_{k})v_{ij}\Big)^{2}}\\
 & +\epsilon_{0}\sqrt{\sum_{i=1}^{n}\sum_{j=1}^{K}v_{ij}^{2}}
\end{align*}
Note that the first term is the $\ell_{2}$-norm of the sum of the
vectors $v_{1},\ldots,v_{n}$ weighted by $\eta_{i}(x_{k})$ or $\eta_{i}(x_{k-1})$.
By assumption, with probability at least $1-\delta^{\prime}$
\[
\Big\Vert\sum_{i=1}^{n}\eta_{i}(x)v_{i}\Big\Vert_{2}\le\epsilon(\delta')+\epsilon_{0}\sqrt{\sum_{i=1}^{n}\|v_{i}\|_{2}^{2}}.
\]
Thus, 
\[
\sup_{x\in\RR}\Big\Vert\sum_{i=1}^{n}\eta_{i}(x)v_{i}\Big\Vert_{2}\le\epsilon(\delta')+\epsilon_{0}\sqrt{\sum_{i=1}^{n}\|v_{i}\|_{2}^{2}},
\]
which completes the proof. 
\end{proof}

\subsection{From Pointwise to Uniform bound for the distribution residual}

\begin{lemma}[From Pointwise to Uniform Error Bound]
\label{lem:pointwise_to_uniform_cdf}
Let $F: \mathbb{R} \to [0, 1]$ be a cumulative distribution function, and let $\widehat{F}: \mathbb{R} \to [0, 1]$ be a non-decreasing, monotonic estimator. Suppose that for any fixed $x \in \mathbb{R}$ and any failure probability $\delta' \in (0, 1)$, $\widehat{F}$ satisfies the pointwise bound
\[
\mathbb{P}\left( |\widehat{F}(x) - F(x)| \le \epsilon(\delta') \right) \ge 1 - \delta',
\]
where $\epsilon(\cdot)$ is a non-increasing function of the confidence level.
Then, for any choice of $M \in \mathbb{N}_{\ge 2}$ and any target failure probability $\delta \in (0, 1)$, setting $\delta' = \frac{\delta}{M - 1}$, the uniform bound holds:
\[
\mathbb{P}\left( \sup_{x \in \mathbb{R}} |\widehat{F}(x) - F(x)| \le \epsilon\left(\frac{\delta}{M-1}\right) + \frac{1}{M} \right) \ge 1 - \delta.
\]
\end{lemma}

\begin{proof}
The proof proceeds in three steps: grid construction, applying the union bound, and sandwiching arbitrary points between grid nodes.

\paragraph{Step 1: Grid Construction.}
Set $\epsilon_0 = \frac{1}{M}$. Construct a finite grid of $M + 1$ points $-\infty = x_0 < x_1 < x_2 < \dots < x_M = \infty$ such that
\[
F(x_k) - F(x_{k-1}) \le \epsilon_0 = \frac{1}{M} \quad \text{for all } k \in \{1, 2, \dots, M\}.
\]
\textit{(Note: Such a partition can always be constructed by choosing $x_k = \inf\{ x \in \mathbb{R} : F(x) \ge k/M \}$.)}

\paragraph{Step 2: Uniform Control over Grid Points via Union Bound.}
Consider the finite set of interior grid points $\{x_1, x_2, \dots, x_{M-1}\}$. For a target failure probability $\delta \in (0, 1)$, set $\delta' = \frac{\delta}{M - 1}$.
Define the bad event at a point $x_k$ as $A_k := \left\{ |\widehat{F}(x_k) - F(x_k)| > \epsilon(\delta') \right\}$. By the given pointwise guarantee, $\mathbb{P}(A_k) \le \delta'$. Applying the Union Bound over all $M-1$ interior points yields
\[
\mathbb{P}\left( \bigcup_{k=1}^{M-1} A_k \right) \le \sum_{k=1}^{M-1} \mathbb{P}(A_k) \le (M - 1) \cdot \delta' = \delta.
\]
Thus, with probability at least $1 - \delta$, the good event $E := \bigcap_{k=1}^{M-1} A_k^c$ holds, where
\[
|\widehat{F}(x_k) - F(x_k)| \le \epsilon(\delta') \quad \text{for all } k \in \{1, 2, \dots, M-1\}.
\]
\textit{(At the boundary points $x_0 = -\infty$ and $x_M = \infty$, $\widehat{F}(x_0) = F(x_0) = 0$ and $\widehat{F}(x_M) = F(x_M) = 1$, so the inequality holds trivially for $k=0$ and $k=M$.)}

\paragraph{Step 3: Interpolation / Sandwiching Argument.}
Condition on the event $E$. Take any arbitrary point $x \in \mathbb{R}$. There exists an index $k \in \{1, 2, \dots, M\}$ such that $x \in [x_{k-1}, x_k]$.
Since both $F$ and $\widehat{F}$ are non-decreasing functions, we have:
\[
\widehat{F}(x_{k-1}) \le \widehat{F}(x) \le \widehat{F}(x_k), \quad \text{and} \quad F(x_{k-1}) \le F(x) \le F(x_k).
\]
We bound the difference $\widehat{F}(x) - F(x)$ from above and below:
\begin{enumerate}
    \item \textbf{Upper Bound:}
    \[
    \begin{aligned}
    \widehat{F}(x) - F(x) &\le \widehat{F}(x_k) - F(x_{k-1}) \\
    &= \bigl(\widehat{F}(x_k) - F(x_k)\bigr) + \bigl(F(x_k) - F(x_{k-1})\bigr) \\
    &\le \epsilon(\delta') + \epsilon_0.
    \end{aligned}
    \]
    \item \textbf{Lower Bound:}
    \[
    \begin{aligned}
    \widehat{F}(x) - F(x) &\ge \widehat{F}(x_{k-1}) - F(x_k) \\
    &= \bigl(\widehat{F}(x_{k-1}) - F(x_{k-1})\bigr) - \bigl(F(x_k) - F(x_{k-1})\bigr) \\
    &\ge -\epsilon(\delta') - \epsilon_0.
    \end{aligned}
    \]
\end{enumerate}
Combining both bounds, we obtain
\[
|\widehat{F}(x) - F(x)| \le \epsilon(\delta') + \epsilon_0 = \epsilon\left(\frac{\delta}{M-1}\right) + \frac{1}{M}.
\]
Since this holds for every $x \in \mathbb{R}$, taking the supremum over $x \in \mathbb{R}$ completes the proof.
\end{proof}

\section{Technical Lemmas}

\begin{lemma}[Threshold for logarithmic inequality. Lemma C.6 in \citealp{kim2025learning}]
\label{lem:logt} For $a>1/2$ and $b>e^{2}$, $t\ge4a\left(1+\log\frac{2a\sqrt{b}}{e}\right)$
implies $t\ge a\log bt^{2}$. 
\end{lemma}

\begin{lemma}[Bretagnolle--Huber Inequality] \label{lem:bretagnolle_huber}
Let $\PP_{1}$ and $\PP_{2}$ be two probability
measures on the same measurable space. Then their total variation
distance satisfies: 
\[
1-\text{TV}(\PP_{1},\PP_{2})\ge\frac{1}{2}\exp\left(-\text{KL}(\PP_{1},\PP_{2})\right),
\]
where $\text{TV}(\PP_{1},\PP_{2}):=\sup_{A}|\PP_{1}(A)-\PP_{2}(A)|$
and $\text{KL}(\PP_{1},\PP_{2})$ denotes the Kullback--Leibler
divergence. 
\end{lemma}

\begin{lemma}[Chernoff bound for the sum of nonnegative random variables.]
\label{lem:Bern_bound} 
For any $\epsilon\in(0,1)$,
with probability at least $1-\delta$, the action $a_s$ satisfies,
\begin{equation*}
W_{a,t} \ge\epsilon t,\label{eq:psi_size}
\end{equation*}
for all $t\ge\frac{2K^2}{\left(1-\epsilon\right)^{2}}\log\frac{1}{\delta}$. 
\end{lemma}

\begin{theorem}[Self-Normalized Bound for Vector-Valued Martingales \cite{abbasi2011improved}]
\label{thm:self_normalized_bound}
Let $\{\mathcal{F}_t\}_{t=0}^\infty$ be a filtration. Let $\{\eta_t\}_{t=1}^\infty$ be a real-valued stochastic process such that $\eta_t$ is $\mathcal{F}_t$-measurable and $\eta_t$ is conditionally $R$-sub-Gaussian for some $R \ge 0$, i.e.,
\[
\forall \lambda \in \mathbb{R}, \quad \mathbf{E}\left[ e^{\lambda \eta_t} \mid \mathcal{F}_{t-1} \right] \le \exp\left( \frac{\lambda^2 R^2}{2} \right).
\]
Let $\{X_t\}_{t=1}^\infty$ be an $\mathbb{R}^d$-valued stochastic process such that $X_t$ is $\mathcal{F}_{t-1}$-measurable. Assume that $V$ is a $d \times d$ positive definite matrix. For any $t \ge 0$, define
\[
\overline{V}_t = V + \sum_{s=1}^t X_s X_s^\top, \qquad S_t = \sum_{s=1}^t \eta_s X_s.
\]
Then, for any $\delta > 0$, with probability at least $1 - \delta$, for all $t \ge 0$,
\[
\|S_t\|_{\overline{V}_t^{-1}}^2 \le 2R^2 \log \left( \frac{\det(\overline{V}_t)^{1/2} \det(V)^{-1/2}}{\delta} \right).
\]
\end{theorem}

\begin{lemma}[A dimension-free bound for vector-valued martingales. Lemma C.2 in
\citet{kim2025learning}.]
\label{lem:dimension_free_martingale}
Let $\{\Fcal_s\}_{s=0}^{t}$ be a
filtration and $\{\eta_{s}\}_{s=1}^{t}$ be a real-valued stochastic
process such that $\eta_{s}$ is $\Fcal_{\tau}$-measurable.
Let $\left\{ X_{s}\right\} _{s=1}^{t}$ be an $\RR^{d}$-valued
stochastic process where $X_{s}$ is $\Fcal_0$-measurable.
Assume that $\{\eta_{s}\}_{s=1}^{t}$ are $\sigma$-sub-Gaussian given
$\{\Fcal_s\}_{s=1}^{t}$. Then with probability at least $1-\delta$,
\begin{equation*}
\Big\Vert\sum_{s=1}^{t}\eta_{s}X_{s}\Big\Vert_{2} \le 4 \sigma\sqrt{\sum_{s=1}^{t}\Big\Vert X_{s}\Big\Vert_2^{2} }\sqrt{2\log\frac{4t^{2}}{\delta}}.
\label{eq:etaX_bound}
\end{equation*}
\end{lemma}

\section{Limitations}

Several limitations present promising avenues for future research. 
\begin{enumerate}
    \item Our current theoretical analysis requires a large constant to guarantee simultaneous convergence over all arms.
    Reducing the constant by optimizing the resampling probability $p$ and the exploration rounds via $c$ remains an open problem for scaling the algorithm to environments for practical usage.
    \item While the proposed doubly robust estimator demonstrates fast empirical convergence on the distribution functions, further theoretical analysis is required to validate whether the simultaneous convergence is possible for the dominance score.
    This may require a more refined finite-sample behavior under more general random variables other than the distribution functions.
    \item While our estimator avoids the curse of dimensionality, the computational complexity is $O(Kt)$ in that it requires to compute $\{\widehat{F}_{k,t}(R_{a_s,s})\}_{s=1}^{t}$.
    Alleviating this computational complexity is a significant open problem and we leave it as a future work.
\end{enumerate}

\end{document}